%% file: example_paper.tex
\documentclass{article}

\usepackage{microtype}
\usepackage{graphicx}
\usepackage{subcaption}
\usepackage{booktabs} 
\usepackage[most]{tcolorbox}
\usepackage{enumitem}

\usepackage{hyperref}

\usepackage[accepted]{icml2026}

\usepackage{amsmath}
\usepackage{amssymb}
\usepackage{mathtools}
\usepackage{amsthm}

\usepackage[capitalize,noabbrev]{cleveref}

\theoremstyle{plain}
\newtheorem{theorem}{Theorem}[section]

\newtheorem{corollary}[theorem]{Corollary}
\theoremstyle{definition}

\newtheorem{assumption}[theorem]{Assumption}
\theoremstyle{remark}

\usepackage[textsize=tiny]{todonotes}

\icmltitlerunning{Calibrated Reinforcement Learning for Multi-Attempt Chain-of-Thought}
\input{notations}

\begin{document}

\twocolumn[
  \icmltitle{Learning to Correct: Calibrated Reinforcement Learning for Multi-Attempt Chain-of-Thought}



  \icmlsetsymbol{equal}{*}

    \begin{icmlauthorlist}
    \icmlauthor{Emrullah Ildiz}{equal,umich}
    \icmlauthor{Halil Alperen Gozeten}{equal,umich}
    \icmlauthor{Ege Onur Taga}{umich}
    \icmlauthor{Samet Oymak}{umich,google}
    \end{icmlauthorlist}

  \icmlaffiliation{umich}{University of Michigan, Ann Arbor}
  \icmlaffiliation{google}{Google Research}

  \icmlcorrespondingauthor{M. Emrullah Ildiz}{eildiz@umich.edu}

  \icmlkeywords{Machine Learning, ICML}

  \vskip 0.3in
]



\printAffiliationsAndNotice{}  

\begin{abstract}

State-of-the-art reasoning models utilize long chain-of-thought (CoT) to solve increasingly complex problems using more test-time computation. 
In this work, we explore a long CoT setting where the model makes up to K successive attempts at solving a problem, in which each attempt is allowed to build on earlier ones after the model receives a hard verifier feedback. This motivates RL methods that can harness per-attempt rewards by carefully weighting individual attempts. We study optimizing the Verification@K reward (the model succeeds by the K-th attempt) and show that naively weighing the attempts by their pass/fail results in biased gradients. We introduce \textbf{Calibrated Attempt-Level} (CAL) GRPO by devising a weighing strategy to obtain unbiased gradients while maintaining small variance. Our theory reveals how incorporating per-attempt rewards influence the training and the eventual Verification@K performance. Experiments, baselines, and ablations on synthetic and real data corroborate our theory and the benefits of CAL-GRPO over vanilla GRPO as well as naive weighting. 

\end{abstract}
\input{sections/introduction}
\input{sections/setup}

\input{sections/main_results}

\input{sections/markov_chain}

\input{sections/math_and_maze}

\input{sections/independent_attempt}
\input{sections/related_work}
\input{sections/conclusion}
\bibliography{example_paper}
\bibliographystyle{icml2026}

\newpage
\appendix
\onecolumn
\input{sections/appendix/experimental_details}
\input{sections/appendix/additional_details}

\input{sections/appendix/equivalence_theorem_proof}
\end{document}

%% file: notations.tex
\usepackage{times}
\usepackage{subcaption}
\usepackage{caption}

\usepackage[utf8]{inputenc} 
\usepackage[T1]{fontenc}    
\usepackage{hyperref}       
\usepackage{url}            
\usepackage{booktabs}       
\usepackage{amsfonts,amsmath,amssymb}       
\usepackage{nicefrac}       
\usepackage{microtype}      
\usepackage{xcolor}         
\usepackage{txfonts}
\usepackage{graphicx}
\usepackage{wrapfig,bm,comment, color}
\usepackage{breakurl,epsfig,epsf,fmtcount,semtrans,multirow,boldline}
\usepackage{tcolorbox}
\tcbuselibrary{skins}
\usepackage{tikz}

\definecolor{darkred}{RGB}{150,0,0}
\definecolor{darkredred}{RGB}{0,0,0}
\definecolor{darkgreen}{RGB}{0,150,0}
\definecolor{darkblue}{RGB}{0,0,200}
\hypersetup{colorlinks=true, linkcolor=darkred, citecolor=darkgreen, urlcolor=darkblue}

\newcommand{\beq}{\begin{equation}}
\newcommand{\ba}{\begin{align}}
\newcommand{\ea}{\end{align}}

\newcommand{\eeq}{\end{equation}}
  
\newcommand{\vct}[1]{\bm{#1}}

\newcommand{\mtx}[1]{\bm{#1}}

\newcommand{\E}{\operatorname{\mathbb{E}}}

\newcommand{\V}{{\mtx{V}}}

\newcommand{\Sc}{\mathcal{S}}

\newcommand{\Tc}{\mathcal{T}}

\newcommand{\Yc}{\mathcal{Y}}

\newcommand{\m}{\vct{m}}

\renewcommand{\P}{\operatorname{\mathbb{P}}}

\newcommand{\Var}{\text{Var}}

\newcommand{\rhoVK}{\rho}

\newcommand{\JVK}{J}

%% file: sections/introduction.tex
\section{Introduction}\label{sec:introduction}

The emergence of advanced reasoning models is fundamentally driven by the scaling of test-time compute, where extended CoT traces help solve more complex problems ~\citep{wei2023chainofthoughtpromptingelicitsreasoning,wang2023selfconsistencyimproveschainthought,yao2023treethoughtsdeliberateproblem,aghajohari2025markovianthinkerarchitectureagnosticlinear}. This paradigm effectively shifts the computational burden to the autoregressive generation process, enabling the model to search for new approaches or revise past proposals~\citep{yao2023reactsynergizingreasoningacting}. While traditional CoT is often viewed as a single coherent trace, recent research increasingly approaches reasoning as an iterative, multi-attempt process where models utilize external verifier feedback to build on and correct earlier failures ~\citep{lightman2023letsverifystepstep,cobbe2021trainingverifierssolvemath,dhuliawala2023chainofverificationreduceshallucinationlarge}. However, optimizing these multi-turn interactions through reinforcement learning (RL) poses a critical challenge for credit assignment ~\citep{ahmadian2024basicsrevisitingreinforcestyle}. Trajectory-level rewards fail to leverage dense feedback and can backpropagate success signals to incorrect/uninformative early attempts~\citep{uesato2022solvingmathwordproblems,lightman2023letsverifystepstep,cobbe2021trainingverifierssolvemath,stiennon2022learningsummarizehumanfeedback,ouyang2022traininglanguagemodelsfollow}. In contrast, naive heuristics can introduce systematic biases in policy gradient estimation, motivating a principled framework~\citep{REINFORCE,NIPS1999_464d828b,schulman2018highdimensionalcontinuouscontrolusing}.

In this work, we frame long CoT as a Verification@K (Ver@K) problem: The model makes at most K successive attempts until it solves the problem, and after each attempt, it gets to observe a hard pass/fail feedback. Under this setting, if a trajectory terminates at $\ell$ attempts, the model receives $\ell$ 0-1 rewards and this dense feedback can be used to enhance RL training. This motivates RL algorithms that weigh individual attempts for better policy gradient estimation. We ask:
\begin{description}[leftmargin=0pt, labelsep=0pt]
    \item \textbf{Q:} What RL algorithms efficiently maximize the Ver@K reward and teach the model better correction capabilities?
\end{description}

Specifically, we examine Attempt-Level variations of REINFORCE Leave-One-Out and GRPO where we weigh the gradient contributions at attempt-level rather than trajectory-level. 
An intuitive strategy is weighing the attempts by their individual pass/fail rewards, which we call Naive Attempt-Level (NAL) Algorithm. Interestingly, we establish that NAL over-emphasizes earlier attempts and results in biased policy gradients for the Ver@K objective. As our key technical contribution, we introduce C(alibrated)AL Algorithm by correcting the bias of NAL through a theoretically-grounded weight assignment. 

Remarkably, we demonstrate that CAL achieves the best-of-both worlds: \textbf{(i)} It returns unbiased gradient unlike NAL and \textbf{(ii)} it achieves lower variance than Trajectory-Level (TL) Algorithm, which employs trajectory-level sparse rewards. Importantly, all three variants are instantiations of a generic Attempt-Level method with specific weight assignments as depicted in Figure \ref{fig:verk_tl_nal_cal}. As a remark, the number of backpropagation steps required is the same across all methods. To summarize, we make the following contributions:
\begin{enumerate}[label=\roman*, leftmargin=*]
\item We introduce attempt-level variations of both RLOO and GRPO to facilitate finer-grained credit assignment.
\item We formally derive the policy gradient of the Ver@K objective and introduce CAL as an effective method to optimize it.
\item We establish that CAL corrects the potential bias of NAL and prove that it also achieves strictly lower variance compared to standard TL under mild conditions.
\item We demonstrate that CAL consistently outperforms TL and NAL baselines across a diverse set of reasoning tasks, including MATH, Maze navigation, and structured Markov Chain benchmarks.
\end{enumerate}




%% file: sections/setup.tex
\section{Problem Setup and Preliminaries}\label{sect:setup}
We consider a reinforcement learning with verifiable reward (RLVR) setup where a policy model is trained to solve problems using a \emph{verification-based multi-attempt procedure}. For each problem, the model can attempt a solution up to $K$ times, and an external verifier provides binary feedback after each attempt.
Each trajectory is assigned a single $0/1$ reward depending on whether
the model achieves a correct solution within the first $K$ attempts.

\begin{figure}[t]
  \centering
\includegraphics[width=\linewidth]{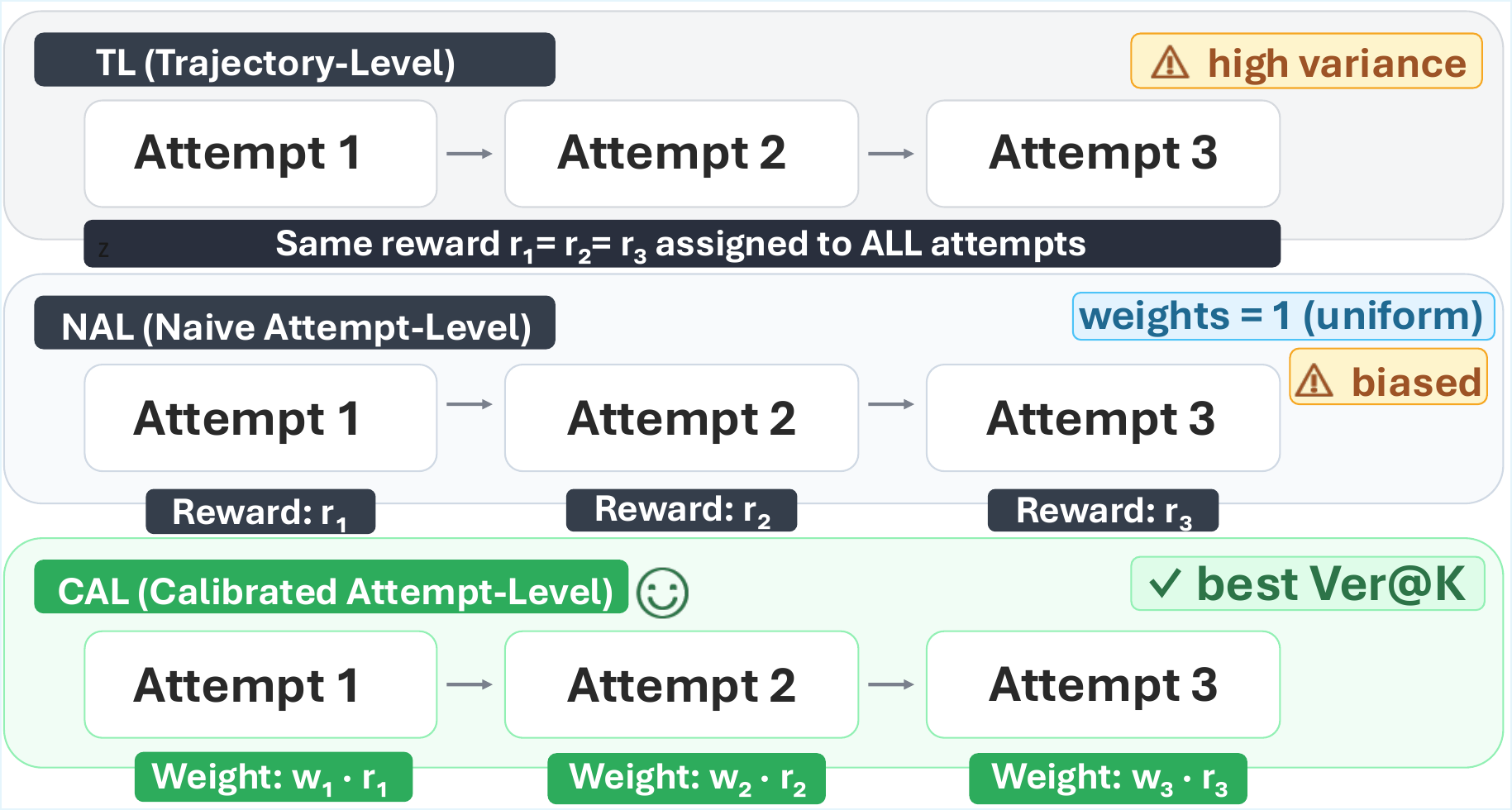}
  \caption{We consider RLOO and GRPO methods that weigh individual attempts. TL is the standard approach that uses the final Ver@K reward for all attempts. NAL uses per-attempt pass/fail as weights. Proposed CAL approach calibrates naive weights to obtain unbiased low-variance gradients.}
  \label{fig:verk_tl_nal_cal}
\end{figure}

Let $\mathcal{X}$ denote the space of problem inputs and
$\Yc$ the space of correct reference answers.
We assume examples $(x,y)$ are drawn from a distribution $P$ over
$\mathcal{X}\times\Yc$. A  policy $\pi_\theta$ (e.g., an autoregressive language model) with parameters $\theta$ maps states to distributions over actions (token sequences).

For a fixed problem $(x,y)$, we consider the following rollout procedure. At attempt $t\in\{1,\dots,K\}$, the environment is in state $s_t$, which encodes the original problem $x$ together with the history of previous attempts and rewards.
The policy samples a candidate solution $
y_t \sim \pi_\theta(\cdot \mid s_{t-1})$,
and an external verifier returns a binary signal
$r_t \in \{0,1\}$ indicating whether $y_t$ is correct with respect to $y$.
If $r_t = 1$, the rollout terminates immediately.
If $r_t = 0$ and $t < K$, the incorrect solution $y_t$ and feedback are appended
to the state, and the process continues with attempt $t+1$.
If all $K$ attempts are incorrect, the rollout terminates after $K$ attempts.

Formally speaking, let the first state $s_0$ be the problem $x$ and let $s_t = (s_{t-1},~y_{t},~r_{t})$ be the state denoting the history up to and including the $(t)$-th attempt, and let $\tau:=s_\Tc= (x,~y_1,~r_1~\cdots~y_{\Tc(\tau)}, r_{\Tc(\tau)})$ denote the resulting random trajectory, where the stopping time ${\Tc(\tau)}\in\{1,\dots,K\}$ is the attempt index at which the rollout stops: ${\Tc(\tau)}$ is the smallest index $t \leq K$ with $r_t=1$ if the model ever succeeds, and ${\Tc(\tau)}=K$ if the model never succeeds. We define the \emph{Ver@K reward} for trajectory $\tau$ as
$
r_{\Tc(\tau)}
$. For notational simplicity, we omit the trajectory $\tau$, so we use $\Tc$ instead of ${\Tc(\tau)}$ in this section.

For fixed $(x,y)$, let $p_\theta(\tau \mid x,y)$ denote the trajectory
distribution induced by the policy $\pi_\theta$ and the verifier. For notational simplicity we write $p_\theta(\tau \mid x)$ with the dependence on 
$y$ through the verifier implicit. The per-example Ver@K success probability is
\begin{align*}
\rhoVK_\theta(x,y)
&:=
\E_{\tau \sim p_\theta(\cdot \mid x)}
\big[ r_\Tc \big] .
\end{align*}
Our training objective is to maximize the expected
Ver@K success probability:
\begin{equation}
\JVK(\theta)
:= \E_{(x,y)\sim P}
\big[ \rhoVK_\theta(x,y) \big].
\label{eq:global-ver-k-objective}
\end{equation}
In this paper, we develop two different approaches to the Ver@K problem, namely trajectory-level optimization and attempt-level optimization. Next, we derive gradients for these approaches to provide insight into their behavior. In the remainder of the paper, we utilize these gradients to maximize the objective of \eqref{eq:global-ver-k-objective} with the REINFORCE algorithm \citep{REINFORCE} and its extension to the leave-one-out version \citep{ahmadian2024basicsrevisitingreinforcestyle}. 

\textbf{Trajectory-level Optimization: }
In this approach, we treat each trajectory $\tau$ as a single solution sampled from the trajectory distribution $p_\theta(\cdot \mid x)$. This formulation allows us to view different trajectories as independent samples from the same distribution. Applying the log-derivative property, at the trajectory level, the gradient of \eqref{eq:global-ver-k-objective} for a fixed problem $(x,y)$ can be written as
\begin{equation}\label{eq:ver-k-reinforce-general}
\begin{aligned}
\nabla_\theta \rhoVK_\theta(x,y)
&= \nabla_\theta
\E_{\tau \sim p_\theta(\cdot\mid x)}
\big[ r_\Tc \big] \\
&= \E_{\tau \sim p_\theta(\cdot\mid x)}
\big[
r_\Tc\,
\nabla_\theta \log p_\theta(\tau\mid x)
\big]. 
\end{aligned}
\end{equation}

The gradient is therefore an expectation over trajectories $\tau$ drawn from $p_\theta(\cdot \mid x)$ and depends on each trajectory only through its realized reward and log-probability.

\textbf{Attempt-level Optimization: }In this approach, we decompose a single trajectory into a sequence of dependent attempts and analyze each attempt individually to derive the gradient of \eqref{eq:global-ver-k-objective} for a fixed problem instance $(x,y)$. This formulation enables attempt-level credit assignment during solution generation, improving optimization efficiency by reducing the sparsity of trajectory-level rewards. On the other hand, because the attempts are sequentially dependent, this approach requires computing more complex conditional probability distributions. The trajectory distribution (over valid trajectories) factorizes as $
p_\theta(\tau \mid x)
=
\prod_{t=1}^{\Tc}
\pi_\theta(y_t \mid s_{t-1}),
$
and therefore
\begin{equation}
\nabla_\theta \log p_\theta(\tau \mid x)
=
\sum_{t=1}^{\Tc} \nabla_\theta \log \pi_\theta(y_t \mid s_{t-1}).
\label{eq:logtraj-grad-det}
\end{equation}
Substituting \eqref{eq:logtraj-grad-det} into \eqref{eq:ver-k-reinforce-general}
and extending the random
upper limit $\Tc$ to $K$ using the indicator $\mathbf{1}_{\Tc\ge i}$ yields
\begin{equation}
\nabla_\theta \rhoVK_\theta(x,y)
=
\sum_{i=1}^{K}
\E_{\tau \sim p_\theta(\cdot\mid x)}
\left[
\mathbf{1}_{\Tc\ge i}
r_\Tc\,
\nabla_\theta \log \pi_\theta(y_i \mid s_{i-1})
\right]. \nonumber
\end{equation}
Next, define the continuation success probability
\begin{equation}
V_{\theta,i}(s_i)
:=
\E\left[r_\Tc\mid s_i\right]
=
\P\left(\exists \,t\in\{i+1,\dots,K\}:r_t=1 \mid s_{i}\right)
\nonumber
\end{equation}
and $V_{\theta,K+1}(\cdot) := 0$; 
where $s_i$ is the state before attempt $i$. Since the verifier is
deterministic, conditioning on $(s_i,y_i)$ fixes $r_i$, and the remaining
Ver@K reward satisfies the one-step recursion
\begin{equation}
\E\left[r_\Tc\mid s_{i-1},y_i\right]
=
r_i + (1-r_i)\,V_{\theta,i}(s_{i}),
\nonumber
\end{equation}
where $s_{i}=(s_{i-1},y_i,r_i)$ is the next state. Applying the tower property to
$\nabla_\theta \rhoVK_\theta(x,y)$ gives
\begin{equation}\label{eqn:gradientAL}
\begin{aligned}
\nabla_\theta \rhoVK_\theta(x,y)
&=
\sum_{i=1}^{K}
\E_{\tau \sim p_\theta(\cdot\mid x)}
\big[
\mathbf{1}_{\Tc \ge i}\,
\big(r_i + (1-r_i)\,V_{\theta,i}(s_{i})\big) \\
&\qquad\qquad\qquad\qquad
\nabla_\theta \log \pi_\theta(y_i \mid s_{i-1})
\big].
\end{aligned}
\end{equation}
In attempt-level optimization, the gradient is expressed as an expectation over attempts and depends on the continuation success probabilities \(V_{\theta,i}(s_{i})\).

Trajectory-level optimization requires sampling full trajectories from the underlying distribution, whereas attempt-level optimization involves computing more complex conditional probability distributions. In this paper, our goal is to develop a practical algorithm that strikes an effective trade-off between these two approaches.

In the remainder of the paper, we describe the REINFORCE algorithm with leave-one-out for both trajectory-level and attempt-level optimization in Section~\ref{sect:main_results}. We then present a practical estimator for $V_{\theta, i}(s_i)$ and introduce an algorithm called \emph{Calibrated Attempt-Level Algorithm (CAL)}. In Section~\ref{sec:experiments}, we compare the proposed method with existing approaches on a synthetic Markov chain benchmark, the MATH dataset, and a maze navigation task. Finally, in Section~\ref{sect:independent_attemtp}, we prove the unbiasedness of the proposed gradient estimator and provide variance-reduction guarantees relative to trajectory-level optimization under mild assumptions.

%% file: sections/main_results.tex
\section{Derivation of RLOO Algorithms}\label{sect:main_results}
In this section, we derive the REINFORCE with leave-one-out (RLOO) algorithm for trajectory-level and attempt-level optimization using the gradients provided in Section \ref{sect:setup}, and we study the relationship between these two approaches. Finally, inspired by these two approaches, we define a practical reward shaping function over a sampled group to apply a GRPO-based algorithm that maximizes the Verification@K reward.

We start by defining the REINFORCE with Leave-one-out (RLOO) \cite{ahmadian2024basicsrevisitingreinforcestyle} algorithm for the trajectory and attempt-level optimization. While defining, we fix a problem instance $(x,y)$ and draw a \emph{group} of $N\geq2$ independent
trajectories
$\tau_1,\dots,\tau_N \overset{\text{i.i.d.}}{\sim} p_\theta(\cdot\mid x)$.

\textbf{RLOO for Trajectory-level Optimization: }In this approach, we assign a single reward to each trajectory in order to apply the RLOO algorithm for the trajectory-level optimization. If there exists one correct answer inside the trajectory, we assign 1 reward; otherwise we assign 0 reward. Formally speaking, let $\Tc(\tau_1),\dots,\Tc(\tau_N)$ be the corresponding stopping times for the trajectories. Then, the assigned reward for the $n$th trajectory is $r_n^\text{TL} := r_{\Tc(\tau_n)}$, where $r_{\Tc(\tau_n)}$ is the verification@K reward defined in Section \ref{sect:setup}. REINFORCE with Leave-One-Out (RLOO) replaces the trajectory reward $r_n^{\text{TL}}$ by the
leave-one-out advantage
\begin{equation}
A_n^{\text{TL}}
~=~
r_n^{\text{TL}} \;-\; \frac{1}{N-1}\sum_{m\neq n} r_m^{\text{TL}},
\nonumber
\end{equation}
and uses the trajectory-level gradient estimator
\begin{equation}\label{defn:TL}
\widehat{G}_{\mathrm{RLOO}}^{\text{TL}}(\theta; x,y)
~=~
\frac{1}{N}\sum_{n=1}^N
A_n^{\text{TL}} \, \nabla_\theta \log p_\theta(\tau_n\mid x).
\end{equation}

\textbf{RLOO for Attempt-level Optimization: }In this approach, we assign a reward to each attempt (instead of assigning a single reward to the whole trajectory!) using the probabilities $V_{\theta, i}(s_i)$ and $\P(\Tc \geq i)$. Let $r_{n,i}$ be the correctness of $n$th trajectory $i$th attempt for $n \in [N]$ and $i \in [\Tc(\tau_n)]$. Inspired by the gradient derivation for the attempt-level optimization in \eqref{eqn:gradientAL}, we reshape this reward to obtain a reward for the attempt-level optimization as follows:
\begin{equation}\label{eqn:reward_defn_ALO}
    r_{n,i}^{\text{AL}} =  \big(r_{n,i} + (1-r_{n,i}) V_{\theta, i}(s_{n,i})\big)
\end{equation}
Now, we define the advantage for each attempt $i$. The important difference from the previous approach is that we apply the leave-one-out centering for each attempt separately. Let $\Sc_i := \{n : n \in [N],~ \Tc(\tau_n) \geq i\}$. Then, when $|\Sc_i| > 1$
\begin{equation}\label{eqn:advantage_defn}
    A_{n,i}^{\mathrm{AL}}
~:=~
r_{n,i}^{\text{AL}}-\frac{1}{|\Sc_i| - 1} \sum_{m\neq n,  m \in \Sc_i} r_{m,i}^{\text{AL}},
\end{equation}
and when $|\Sc_i| = 1$, $A_{n,i}^{\mathrm{AL}} = 0$. Based on these advantages, we define the attempt-level gradient estimator as 
\begin{equation}\label{defn:AL}
    \widehat{G}_{\mathrm{RLOO}}^{\mathrm{AL}}(\theta;x,y)
:=
\sum_{i=1}^K \frac{1}{N}\sum_{n \in \Sc_i} A_{n,i}^{\mathrm{AL}} \nabla_\theta\log\pi_\theta(y_{n,i}\mid s_{n,i-1}).
\end{equation}

Until now, we have defined the RLOO algorithm for trajectory-level and attempt-level optimization. 
Now, we are ready to state our main result in this section:
\begin{theorem}\label{thm:main_theorem}
Recall the definitions of trajectory and attempt-level estimators $\widehat{G}_{\mathrm{RLOO}}^{\text{TL}}(\theta; x,y)$ and $\widehat{G}_{\mathrm{RLOO}}^{\text{AL}}(\theta; x,y)$ in \eqref{defn:TL} and \eqref{defn:AL}, respectively. Then, we have
\begin{equation}
    \E\big[\widehat{G}_{\mathrm{RLOO}}^{\text{TL}}(\theta; x,y)\big] =  \E\big[\widehat{G}_{\mathrm{RLOO}}^{\text{AL}}(\theta; x,y)\big] = \nabla_\theta \rhoVK_\theta(x,y). \nonumber
\end{equation}
Define $g_{n,i}\;:=\;\mathbf{1}_{\Tc_n\ge i}\,\nabla_\theta\log\pi_\theta(y_{n,i}\mid s_{n,i-1})$ and 
\begin{equation}\label{defn:Z_i}
\begin{aligned}
    Z_i^{\mathrm{TL}}:=\frac{1}{N}\sum_{n=1}^N A_n^{\mathrm{TL}}\,g_{n,i},
~
~
Z_i^{\mathrm{AL}}:=\frac{1}{N}\sum_{n=1}^N A_{n,i}^{\mathrm{AL}}\,g_{n,i}.
\end{aligned}
\end{equation}Then, we have  
\begin{equation}
    \Var\big(Z_i^{\mathrm{TL}}\big) \geq \Var\big(Z_i^{\mathrm{AL}}\big) \qquad \forall i \in [K] \nonumber
\end{equation}
where the equality occurs only when $K=1$ or $V_{\theta, i}(s_i) \in \{0, 1\}$ for all possible $s_i$. 

\end{theorem}
This theorem yields two conclusions. First, the RLOO gradient estimators for trajectory-level and attempt-level optimization are unbiased. The key reason is that we define the rewards $r_i^{\text{TL}}$ and $r_{n,i}^{\text{AL}}$ to match the gradients derived in \eqref{eq:ver-k-reinforce-general} and \eqref{eqn:gradientAL}. Second, the attempt-level gradient estimator has lower variance than the trajectory-level gradient estimator, because it exploits the rewards from individual attempts rather than treating the entire trajectory as a single sample. Under reasonable assumptions on the attempt-level components, we show the variance reduction for each attempt. Under suitable cross-covariance conditions between the attempts' gradients, it implies that $\Var(\widehat{G}_{\mathrm{RLOO}}^{\text{TL}}) \geq \Var( \widehat{G}_{\mathrm{RLOO}}^{\text{AL}})$. The proof for the above theorem is provided in Appendix \ref{app:main_results}.

The main practical drawback of attempt-level optimization is that we assume access to the probabilities $V_{\theta,i}(s_{n,i})$ for every possible trajectory. In practice, it is difficult to compute these probabilities for all trajectories. To overcome this difficulty, we provide a rough estimate of these probabilities. We refer to this algorithm as the Calibrated Attempt-level algorithm.

\begin{figure*}[t!]
    \centering
    \begin{subfigure}[t]{0.32\textwidth}
        \centering
        \includegraphics[width=\linewidth]{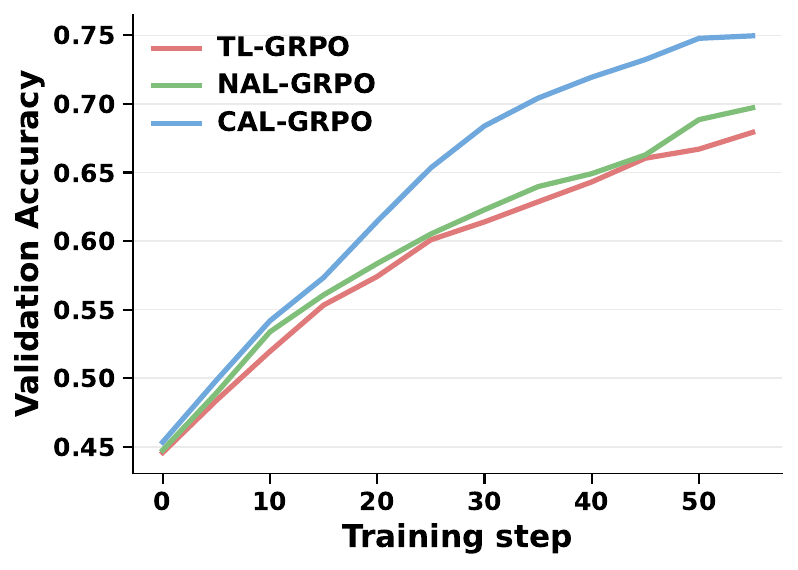}
        \caption{\textbf{Ver@2 validation accuracy on MATH} across training. CAL-GRPO attains the highest Ver@2, with NAL-GRPO improving over TL-GRPO.}
        \label{fig:math_ver2_acc}
    \end{subfigure}
    \hfill
    \begin{subfigure}[t]{0.32\textwidth}
        \centering
        \includegraphics[width=\linewidth]{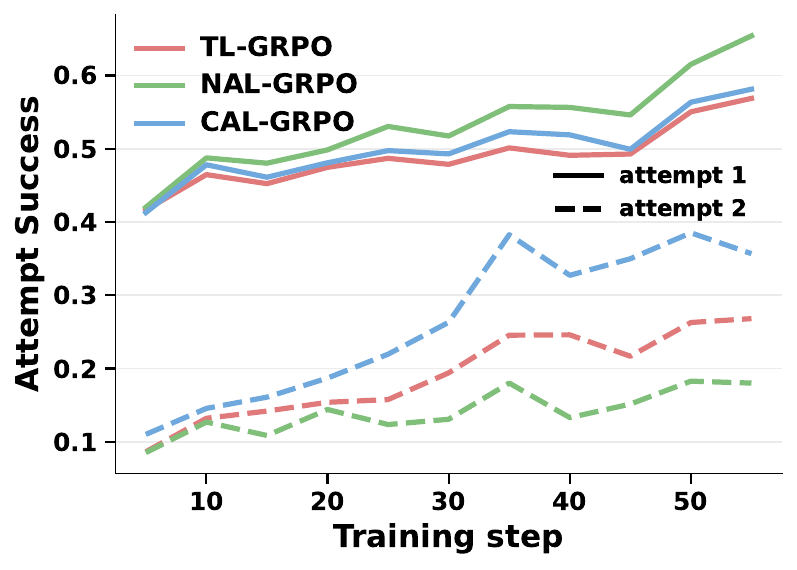}
        \caption{\textbf{Attempt-wise success.} Solid: attempt-1 success. Dashed: attempt-2 success conditioned on attempt 1 failing; CAL-GRPO notably improves corrections on attempt-2.}
        \label{fig:math_ver2_attempt_success}
    \end{subfigure}
    \hfill
    \begin{subfigure}[t]{0.32\textwidth}
        \centering
        \includegraphics[width=\linewidth]{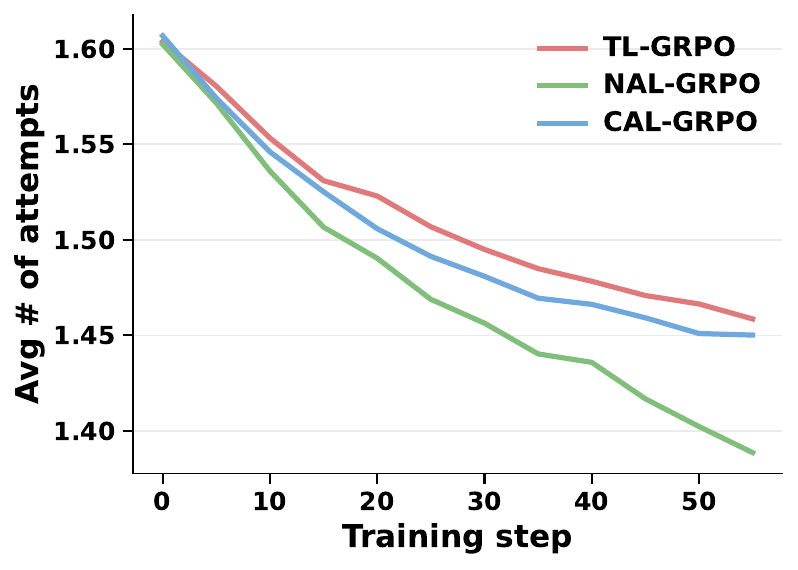}
        \caption{\textbf{Average attempts per problem} under Ver@2. NAL-GRPO reduces attempts the most, while CAL-GRPO trades slightly higher attempts for higher Ver@2.}
        \label{fig:math_ver2_avg_attempts}
    \end{subfigure}

    \caption{\textbf{Optimizing the Ver@2 objective on MATH with GRPO variants.} We compare Trajectory-Level GRPO (TL-GRPO), Naive Attempt-Level GRPO (NAL-GRPO), and Calibrated Attempt-Level GRPO (CAL-GRPO). \textbf{(a)} CAL-GRPO achieves the best Ver@2 over training. \textbf{(b)} The improvement is primarily driven by stronger second-attempt correction, whereas NAL-GRPO emphasizes first-attempt success. \textbf{(c)} NAL-GRPO minimizes the number of attempts, while CAL-GRPO uses slightly more attempts but delivers higher overall Ver@2.}
    \label{fig:math_ver2_estimators}
\end{figure*}

\textbf{Calibrated Attempt-level Algorithm (CAL): }In this algorithm, we fix a problem $(x,y)$ and draw a \emph{group} of $N$ independent trajectories $\tau_1,\dots,\tau_N \overset{\text{i.i.d.}}{\sim} p_\theta(\cdot\mid x)$. Recall that $r_{n,i}$ represents the correctness of the output for the $n$th trajectory at $i$th attempt and recall the definition of $\Sc_i = \{n : n \in [N],~ \Tc(\tau_n) \geq i\}$. We define an estimate of the success probability for the attempt $i \in [\Tc(\tau_n)]$ with leave-one-out as
\begin{equation}
    \hat{\rho}_{n,i} := \frac{\sum_{m \in \Sc_i, m\neq n} \mathbf{1}_{r_{m,i} = 1}}{|\Sc_i|-1}. \nonumber
\end{equation} 
Then, inspired by the reward definition $r_{\theta, i}^\text{AL}$ in \eqref{eqn:reward_defn_ALO}, we assign the following weight to the attempt $i \in [K]$:
\begin{equation}\label{eqn:defn_weight_WAL}
    w_{n,i}^{\text{CAL}} =  \Pi_{j = i+1}^{K} (1-\hat{\rho}_{n,j}),
\end{equation}
 For the reward $r_{n,i}^{\text{CAL}} := r_{n, i}$ and define $A_{n,i}^{\text{CAL}}$ accordingly similar to \eqref{eqn:advantage_defn}. Then, the gradient estimator for the weighted attempt-level algorithm is the following:
\begin{equation}\label{eqn:GRLOO_WAL}
\begin{aligned}
&\widehat{G}_{\mathrm{RLOO}}^{\mathrm{CAL}}(\theta;x,y)
 \\
 & \qquad :=\sum_{i=1}^K \frac{1}{N}\sum_{n \in \Sc_i} w_{n,i}^{\text{CAL}} A_{n,i}^{\mathrm{CAL}} \nabla_\theta\log\pi_\theta(y_{n,i}\mid s_{n,i-1}).
\end{aligned}
\end{equation}
Note that we do not utilize the probability of $\V_{\theta, i}(s_i)$ in the definition of \eqref{eqn:GRLOO_WAL} compared to \eqref{eqn:reward_defn_ALO} as we estimate the future success probabilities based on the samples inside the sampled. 
In addition to this algorithm, we also defined two different baselines.

\textbf{Naive Attempt-level Algorithm (NAL): }In this algorithm, we assign a reward to each attempt in a trajectory, but we do not apply the weight that we defined in \eqref{eqn:defn_weight_WAL}. Instead, we weight with $1$ and obtain the gradient estimator for the attempt-level algorithm similar to \eqref{eqn:GRLOO_WAL}.

\textbf{Trajectory-level Algorithm (TL): }In this algorithm, we directly use the gradient estimator defined in \eqref{defn:TL}.

An important remark is that the number of backpropagation steps required is the same across all methods. Compared to the TL algorithm, the NAL and CAL algorithms assign rewards to individual attempts rather than a single reward to the entire trajectory, which reduces the variance of the gradient estimator (Theorem \ref{thm:main_theorem}). The key distinction between the two is that the CAL algorithm assigns attempt-specific weights based on future success probabilities, inspired by \eqref{eqn:reward_defn_ALO}, whereas the NAL algorithm uses uniform weights. A central contribution of this paper is this calibrated attempt-wise reward assignment, which avoids treating each trajectory as a single output generation. In our experiments, the NAL algorithm serves as a natural ablation, representing a naive uniform-weight variant of the proposed CAL algorithm.

%% file: sections/markov_chain.tex
\section{Experiments on Ver@K}\label{sec:experiments}

In this section, we evaluate TL-GRPO, NAL-GRPO, and CAL-GRPO on three benchmarks: a synthetic Markov chain task, MATH, and Maze. We first describe the experimental setup and then present the main results and analysis on Ver@K.

\subsection{Experimental Setup}
We compare three algorithms, namely trajectory-level, naive attempt-level, and calibrated attempt-level on three datasets: a synthetic Markov chain task, the MATH dataset \citep{hendrycks2021measuringmathematicalproblemsolving}, and the Maze navigation task \citep{ivanitskiy2023configurablelibrarygeneratingmanipulating}. All algorithms are implemented using Group Relative Policy Optimization (GRPO)~\cite{shao2024deepseekmathpushinglimitsmathematical} with the advantage shaping methods described in Section~\ref{sect:main_results}. Different from the previous section’s leave-one-out formulation, in our experiments we use the standard GRPO group normalization: advantages are computed by centering rewards with the within-group empirical mean and scaling by the within-group empirical standard deviation; for NAL-GRPO and CAL-GRPO this normalization is applied separately at each attempt index over trajectories that reach that attempt. We use this normalization throughout all our experiments as we empirically observe that dividing by the within-group empirical standard deviation improves optimization, and Appendix~\ref{app:add_results} reports an ablation without the standard deviation scaling. Following the notation in Section~\ref{sect:main_results}, in this section we refer to these algorithms as TL-GRPO, NAL-GRPO, and CAL-GRPO, respectively. Further benchmark, verifier, rollout, and GRPO-variant details are
provided in \crefrange{app:datasets}{app:verk_estimators}.

For the MATH and Maze tasks, we implement Ver@K training by extending the open-source verl framework~\cite{Sheng_2025} and leveraging its multi-attempt rollout and interaction framework for iterative verifier feedback. Specifically, we use a multi-attempt interaction loop (problem $\rightarrow$ attempt $\rightarrow$ verifier feedback) with $K \in {2,4}$ (Ver@2 and Ver@4), while keeping rollout and group-sampling configurations fixed across all estimators. For the Maze Ver@2 results on $5{\times}5$ and $9{\times}9$, we report mean curves over five independent training runs with different random seeds, together with confidence intervals across runs. For the remaining settings, we report single-run curves.


\begin{figure*}[t!]
    \centering
    \begin{subfigure}[t]{0.4\textwidth}
        \centering
        \includegraphics[width=\linewidth]{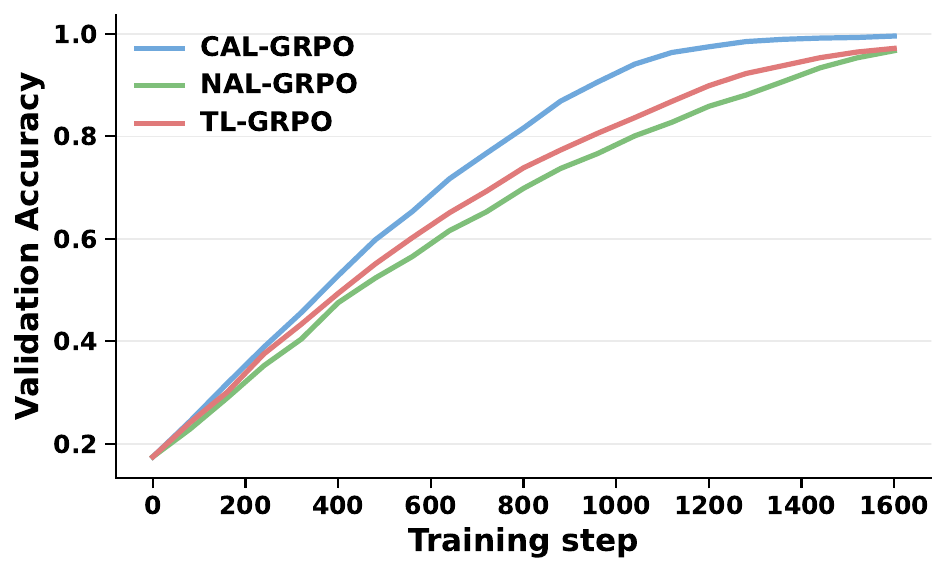}
        \caption{\textbf{No-trap Markov chain Ver@4} ($n_{\text{hubs}}{=}5$, $m{=}6$).
        Validation success vs.\ training step; the calibrated attempt-level method learns fastest and reaches the highest success.}
        \label{fig:mc_notrap}
    \end{subfigure}\hspace{0.06\textwidth}%
    \begin{subfigure}[t]{0.4\textwidth}
        \centering
        \includegraphics[width=\linewidth]{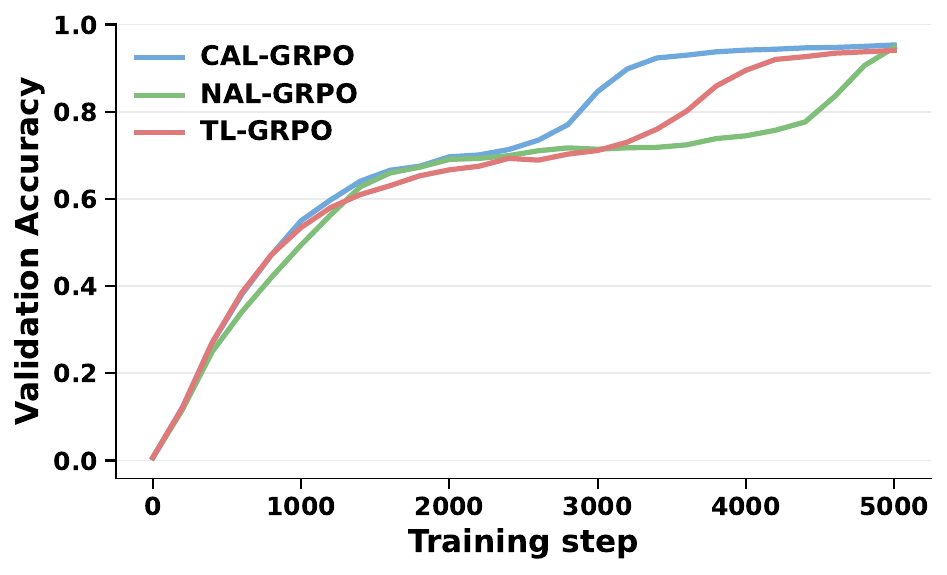}
        \caption{\textbf{Trap Markov chain Ver@2} ($n_{\text{hubs}}{=}5$, $m{=}6$).
        Validation success vs.\ training step with a sampled pre-terminal trap to avoid; the calibrated attempt-level method reaches higher accuracy earlier.}
        \label{fig:mc_trap}
    \end{subfigure}

    \caption{\textbf{Optimizing the Ver@4 and Ver@2 objectives on Markov Chain task with GRPO variants.} We compare TL-GRPO, NAL-GRPO, and CAL-GRPO. Success is defined by reaching the absorbing terminal state with a valid shortest-length trajectory; the trap setting in (b) additionally fails any attempt that visits the trap state. CAL-GRPO consistently reaches higher accuracies in both settings and breaks the plateau earlier than TL-GRPO and NAL-GRPO in the trap setting.}
    \label{fig:mc_results}
\end{figure*}

%% file: sections/math_and_maze.tex
\textbf{MATH Dataset.}
We train and evaluate on the MATH dataset~\cite{hendrycks2021measuringmathematicalproblemsolving}, which contains $12{,}500$ competition-style problems ($7{,}500$ train / $5{,}000$ test). We use the official train split for RL training and report Ver@K on the held-out test split. The verifier extracts the model's final answer and compares it to the ground-truth answer (reward $r_t{=}1$ if correct, else $0$). When an attempt is incorrect, we append a short, fixed-format feedback message indicating failure and prompt the model to revise its answer and try again; otherwise the interaction terminates immediately on success.

\textbf{Maze Task.}
We consider a text-based maze navigation task which is a standard benchmark of planning and sequential decision-making for LLMs~\cite{dao2025alphamazeenhancinglargelanguage,chen2025passktrainingadaptivelybalancing, ivanitskiy2023configurablelibrarygeneratingmanipulating}. We generate synthetic datasets using the configurable \texttt{maze-dataset} library~\cite{ivanitskiy2023configurablelibrarygeneratingmanipulating}. For maze sizes $9{\times}9$, $7{\times}7$, and $5{\times}5$, we create $10{,}000$ training and $1{,}000$ test mazes. Each instance provides a maze grid with a designated start and goal. The model outputs a sequence of discrete moves of $\{\texttt{U,D,L,R}\}$; and we deterministically check whether the sequence constitutes a valid path from start to goal without crossing walls, and return $r_t\in\{0,1\}$. We train with the same Ver@K procedure and the same three GRPO estimators as in the previous sections, and report Ver@K on the corresponding test sets for each maze size. We provide further experimental details in Appendix~\ref{app:exp_details}.

\textbf{Markov Chain Task.}
We introduce a synthetic structured Markov chain task for controlled evaluation of multi-attempt reward assignment under a deterministic verifier. The state space is organized into $n_{\text{hubs}}$ ordered hubs, each containing $m$ states, so that $n_{\text{states}} = n_{\text{hubs}}\cdot m$. Transitions are predominantly local within a hub (moves to nearby states, e.g., within $\pm 1$ positions), with additional sparse \emph{boundary} transitions that connect consecutive hubs. The final state in the last hub is an absorbing terminal state. Each instance samples a start state $s_0$ \emph{from the set of states that can reach the terminal}. In each attempt, the model autoregressively proposes a candidate state sequence. The verifier returns $r_t=1$ iff the proposed trajectory (i) reaches the absorbing terminal state, (ii) uses only valid transitions in the Markov chain transition graph, (iii) the state id never decreases, and (iv) has length equal to the precomputed shortest-path distance from $s_0$ to the terminal. We also consider a \emph{trap} variant, where a trap state is sampled among two pre-terminal states in the last hub with $1/2$ probability. Visiting the trap causes immediate failure, and shortest-path distances are computed on the graph with the trap removed. Under Ver@K, failed attempts reset to the same start state, and the model retries for up to $K$ attempts with early stopping. The trapped version needs Ver@K ($K \geq 2)$ approach to solve the problem as the trap is randomly selected for each trajectory. 

\textbf{Models.} For MATH and Maze, we train the instruction-tuned Qwen3-4B-Instruct model~\cite{yang2025qwen3technicalreport}. For the Markov chain benchmark, we initialize the policy from an SFT checkpoint trained on synthetic optimal trajectories generated from the known transition graph. Concretely, we sample start states (and also the trap state in the trap setting), compute the corresponding shortest valid trajectory to the absorbing terminal state on the graph with the trap removed when applicable, and use the resulting state-id sequence as the supervised target. We then fine-tune this checkpoint with the same Ver@K GRPO procedure used in the main experiments. \crefrange{app:datasets}{app:training_setup} provide the benchmark, verifier, tokenization, rollout, and optimization details.

\begin{figure*}[t!]
    \centering
    \begin{subfigure}[t]{0.32\textwidth}
        \centering
        \includegraphics[width=\linewidth]{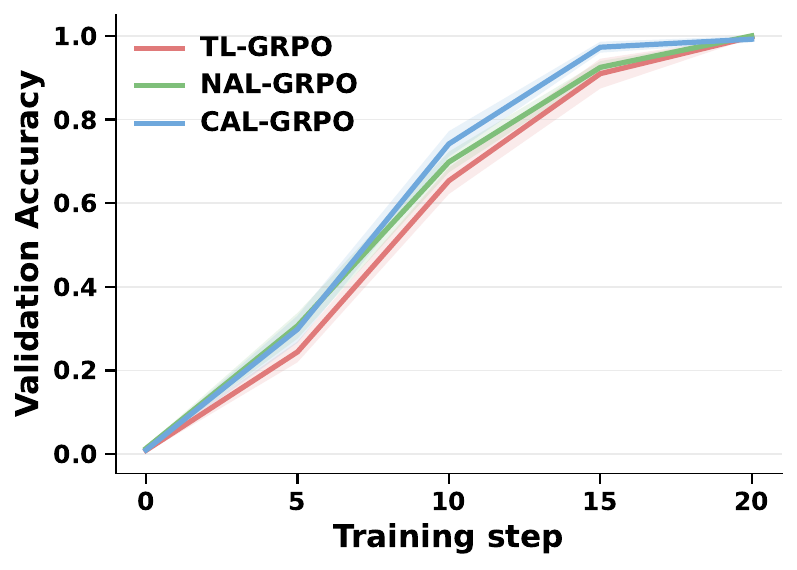}
        \caption{\textbf{Ver@2 validation accuracy on Maze} across training; CAL-GRPO reaches higher accuracy earlier, and all methods converge near-perfect.}
        \label{fig:maze_ver2_acc}
    \end{subfigure}
    \hfill
    \begin{subfigure}[t]{0.32\textwidth}
        \centering
        \includegraphics[width=\linewidth]{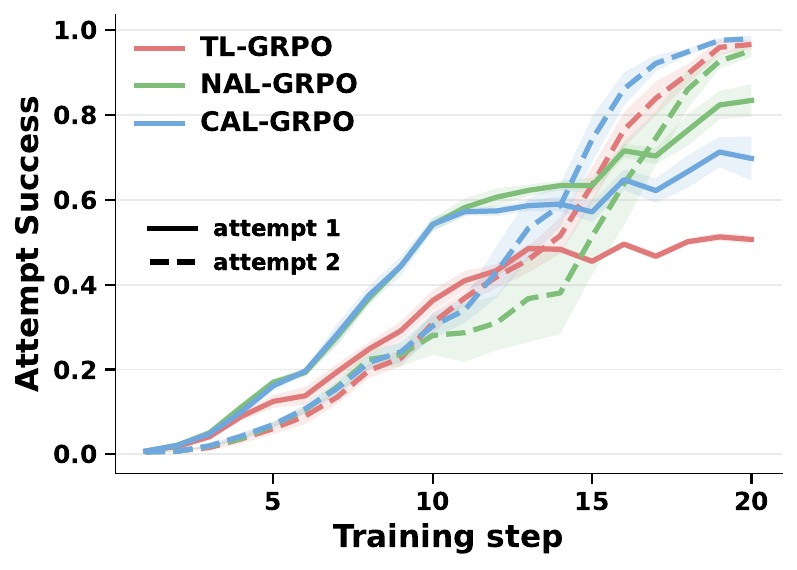}
        \caption{\textbf{Attempt-wise success.} Solid: attempt-1 success. Dashed: attempt-2 success given attempt-1 failure; CAL-GRPO yields the strongest correction on the second attempt.}
        \label{fig:maze_ver2_attempt_success}
    \end{subfigure}
    \hfill
    \begin{subfigure}[t]{0.32\textwidth}
        \centering
        \includegraphics[width=\linewidth]{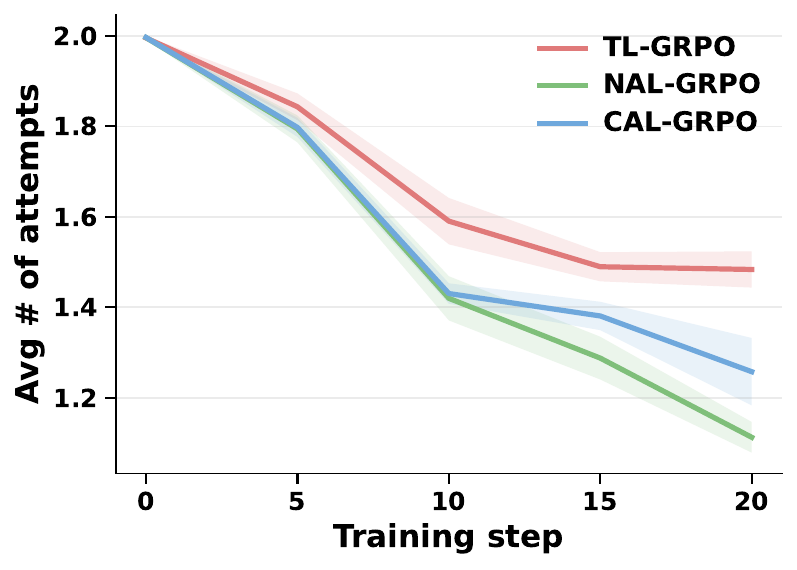}
        \caption{\textbf{Average attempts per instance} under early stopping. NAL-GRPO minimizes the number of attempts, mimicking optimization of a discounted objective.}
        \label{fig:maze_ver2_avg_attempts}
    \end{subfigure}

    \caption{\textbf{Optimizing the Ver@2 objective on 5x5 Maze with GRPO variants.} We compare Trajectory-Level GRPO (TL-GRPO), Naive Attempt-Level GRPO (NAL-GRPO), and Calibrated Attempt-Level GRPO (CAL-GRPO). \textbf{(a)} CAL-GRPO reaches higher Ver@2 accuracy earlier in training. \textbf{(b)} NAL-GRPO increases attempt-1 success, while CAL-GRPO strengthens attempt-2 correction. \textbf{(c)} NAL-GRPO uses the fewest attempts, effectively behaving like optimization of a discounted objective.}
    \label{fig:maze_ver2_estimators}
\end{figure*}

\subsection{Results and Analysis}

We study three questions: (Q1) Does \emph{attempt-level credit assignment} (NAL-GRPO/CAL-GRPO) improve over \emph{trajectory-level credit assignment} (TL-GRPO)? (Q2) Is the calibration in Eq.~(10) necessary beyond naive attempt-level rewards? (Q3) How does calibration change the accuracy--compute tradeoff across attempts? On MATH and Maze, attempt-level credit assignment substantially improves over TL-GRPO. On the Markov chain task, however, naive attempt-level normalization is not uniformly beneficial and can underperform TL-GRPO. Across all three benchmarks, CAL-GRPO is the most robust method because Eq.~(10) reweights each attempt by its estimated marginal contribution to final Ver@K.

\subsubsection{MATH and Maze tasks}
\noindent\textbf{Attempt-level credit assignment improves learning over trajectory-level credit assignment.} Figures~\ref{fig:math_ver2_estimators} and~\ref{fig:maze_ver2_estimators} show that, on MATH and Maze, moving from trajectory-level credit assignment (TL-GRPO) to attempt-level credit assignment (NAL-GRPO/CAL-GRPO) substantially improves learning under Ver@2 and Ver@4 objectives (please see Appendix~\ref{app:add_results} for Ver@4 results). TL-GRPO assigns the same binary trajectory success signal to all tokens, so a trajectory solved only on attempt 2 still provides positive updates to the attempt-1 tokens that produced an incorrect first answer. Under multi-attempt procedure, this mixes gradients for first-attempt correctness and correction-with-feedback, yielding a noisy (and sometimes conflicting) learning signal for improving attempt-1 behavior. In contrast, NAL-GRPO normalizes rewards separately per attempt among trajectories that actually reach that attempt, which yields cleaner reward assignment: attempt 1 is trained on first-attempt correctness, and attempt 2 is trained on correction conditioned on failure. This sharply accelerates attempt-1 success and correspondingly reduces the expected number of attempts executed under early stopping (Figures~\ref{fig:math_ver2_avg_attempts} and~\ref{fig:maze_ver2_avg_attempts}). 

\noindent\textbf{Naive attempt-level normalization induces an early-attempt bias.} However, NAL-GRPO implicitly downweights later attempts as training progresses: because attempt $t{+}1$ is visited only on the shrinking subset of failure cases, the total gradient contribution of later attempts is effectively multiplied by a reach probability, which prioritizes earlier-attempt correctness to minimize attempts, closely aligning with Corollary \ref{prop:discounted-obj}. This indicates NAL-GRPO implicitly optimizes a discounted objective where later steps receive smaller weight once earlier steps become more accurate.

\noindent\textbf{Calibration restores learning signal for later-attempts and improves final Ver@K.} Our \emph{Calibrated Attempt-Level GRPO (CAL-GRPO)} corrects this misalignment by reweighting attempts according to Eq. \eqref{eqn:defn_weight_WAL}, which estimates the probability that all later attempts fail and therefore measures the marginal contribution of attempt $t$ to final Ver@K success. Intuitively, this focuses attempt-1 updates on the hard residual cases where a second attempt is unlikely to correct the mistake anyway, while avoiding over-investing in attempt-1 improvements when attempt 2 already succeeds accurately, where further gains in attempt-1 accuracy have diminishing returns for Ver@2 accuracy. As a result, CAL-GRPO preserves more update budget for learning correction-with-feedback when second-attempt correction still contributes substantially to Ver@2. Empirically, CAL-GRPO yields the best Ver@2 on both tasks (Figures~\ref{fig:math_ver2_acc} and~\ref{fig:maze_ver2_acc}) by substantially increasing conditional second-attempt success (Figures~\ref{fig:math_ver2_attempt_success} and~\ref{fig:maze_ver2_attempt_success}), at the cost of a modest increase in average attempts relative to NAL-GRPO (Figures~\ref{fig:math_ver2_avg_attempts} and~\ref{fig:maze_ver2_avg_attempts}). This indicates that if the primary goal is Ver@K accuracy, the attempt-weighted CAL-GRPO method works well; and if reducing number of attempts matters, NAL-GRPO can be preferable. We provide further analysis and additional Maze results for $9{\times}9$ and $7{\times}7$ under Ver@2 and Ver@4 in Appendix~\ref{app:add_results}.

\subsubsection{Markov chain task}
\noindent\textbf{A different regime under stricter verifier.} Figure~\ref{fig:mc_results} provides a controlled comparison of TL-GRPO, NAL-GRPO, and CAL-GRPO and highlights a qualitatively different regime from MATH/Maze. We observe that NAL-GRPO is the least sample-efficient and underperforms even TL-GRPO in both the no-trap and trap settings. With a strict, sequence-level verifier (valid shortest path or nothing), per-attempt normalization can over-penalize early failed attempts and slow down learning progress. Under NAL-GRPO, attempt-$1$ tokens receive zero/negative advantage whenever attempt 1 fails, \emph{even if the trajectory is “useful” in the sense that it enables a later attempt to succeed after feedback}. On the other hand, TL-GRPO still propagates eventual Ver@K success back to earlier tokens, which can implicitly reward early-attempt behaviors that set up effective correction, and can therefore remain competitive. 

\noindent\textbf{Calibration preserves useful correction signal.} CAL-GRPO consistently outperforms the other methods as it keeps attempt-aware reward assignment while explicitly reweighting updates toward the later attempts that still contribute the most to marginal Ver@K success by effectively preserving learning budget for the correction behavior that drives the remaining improvements.

\noindent\textbf{CAL-GRPO breaks the trap-induced plateau earlier.} In the \emph{trap} variant of the Markov chain task, training exhibits a plateau around $\approx 0.75 = 1 - (1/2)^2$ as trap state is random with 1/2 probability before further gains, which reflects the additional late-stage failure mode introduced near the terminal hub. The policy can learn the general “forward shortest-path” behavior relatively quickly, but accurately avoiding the trap while still satisfying the shortest-path constraint is harder and becomes the dominant error source. CAL-GRPO breaks through this plateau earlier than the other baselines, which matches its intended behavior. Once attempt-1 behavior is already reasonably good, further improvements in Ver@K come from stabilizing correction behavior on retries for the remaining hard cases (here, trap-related failures), rather than continuing to over-optimize the first attempt.

%% file: sections/independent_attempt.tex
\section{Optimization of Verification@K at Attempt Level}\label{sect:independent_attemtp}

In Section~\ref{sect:main_results}, we have compared the trajectory-level optimization and attempt-level optimization, and we have provided a practical estimate of the attempt-level optimization with the weighted attempt-level algorithm. In this section, we prove the unbiasedness of the weighted attempt-level algorithm and establish variance reduction guarantees under the independent attempt index assumption we state below. Furthermore, we provide more insights about the attempt-level algorithm that we described in Section \ref{sect:main_results}. 

The independent attempt-index assumption is the following:
\begin{assumption}\label{assump:index-only-assumption}
    For a fixed problem $(x,y)$ and any attempt index $i\in\{1,\dots,K\}$, we assume that the conditional
distribution of the $i$-th response depends only on $x$ and the number
of previous responses. Formally, for any two histories $s_{1:i-1}$ and $s'_{1:i-1}$ of the
same length,
\begin{align*}
p_\theta(y_i = y \mid s_{1:i-1})
=
p_\theta(y_i = y \mid s'_{1:i-1})
\quad
\text{for all } y.
\end{align*}
\end{assumption}

Assumption \ref{assump:index-only-assumption} implies that there exists a family of policies
$(\pi_\theta^{(i)}(\cdot\mid x))_{i=1}^K$ such that
\begin{equation}
p_\theta(y_i = y \mid s_{1:i-1}) = \pi_\theta^{(i)}(y\mid x),
\quad i=1,\dots,K. \nonumber
\end{equation}
Now, we are ready to state the theorem:
\begin{theorem}\label{thm:independent-attempt_main}
    Suppose that Assumption \ref{assump:index-only-assumption} holds. Recall the definitions of $\widehat{G}_{\mathrm{RLOO}}^{\text{TL}}(\theta; x,y)$ and $\widehat{G}_{\mathrm{RLOO}}^{\mathrm{CAL}}(\theta;x,y)$ in \eqref{defn:TL} and \eqref{eqn:GRLOO_WAL}, respectively. Then, we have the following:
    \begin{align}
        \E\big[\widehat{G}_{\mathrm{RLOO}}^{\mathrm{CAL}}(\theta;x,y)\big] = \nabla_\theta \rhoVK_\theta(x,y). \nonumber
    \end{align}
\end{theorem}
This theorem is a conjugate of Theorem \ref{thm:main_theorem} for the calibrated attempt-level algorithm under Assumption \ref{assump:index-only-assumption}. It shows that the calibrated attempt-level algorithm is unbiased. The proof is provided in Appendix \ref{app:independent_attempt}. Next, we analyze the naive attempt index model. Before, define its single-attempt success probability for each attempt $i$,
\begin{equation}
\rho_{\theta,i}(x,y)
:=
\P\left(r_i=1 \mid x,y, \text{using policy } \pi_\theta^{(i)}\right).
\nonumber
\end{equation}
As before, $r_i\in\{0,1\}$ indicates whether the $i$-th response is
correct with respect to the answer $y$. 
\begin{corollary}\label{prop:discounted-obj}
    The naive attempt-level algorithm optimizes $\sum_{i \in [K]} \gamma_i \P(r_i = 1)$ in expectation where $\frac{\gamma_i}{\gamma_j} = \Pi_{t = j+1}^i (1 - \rho_{\theta, i}(x,y))$ for $i > j$. 
\end{corollary}
This corollary is a direct result of the unbiasedness of $\widehat{G}_{\mathrm{RLOO}}^{\mathrm{CAL}}(\theta;x,y)$ by Theorem \ref{thm:independent-attempt_main} and the definition of $w_{n,i}^{\text{CAL}}$ in \eqref{eqn:defn_weight_WAL}. It shows that the naive attempt-level algorithm optimizes a discounted loss as a function of $K$ instead of the Verification@K loss defined in \eqref{eq:global-ver-k-objective}, which is consistent with Figures \ref{fig:math_ver2_avg_attempts} and \ref{fig:maze_ver2_avg_attempts}. It also shows that the naive attempt-level algorithm is biased whereas the proposed calibrated attempt-level algorithm is shown to be unbiased in Theorem \ref{thm:independent-attempt_main}.

%% file: sections/related_work.tex
\section{Related Work}\label{sec:related_work}
\textbf{Reinforcement learning with verifiable feedback for reasoning LLMs.}
Policy-gradient fine-tuning of reasoning LLMs with automatically verifiable rewards has been actively explored across a range of settings~\citep{Guo_2025,cui2025processreinforcementimplicitrewards,heckel2026asymmetric,wang2025reinforcementlearningreasoninglarge,zhangbread,su2025crossingrewardbridgeexpanding}.
Several analyses further study what such training changes in the model’s solution distribution and how it interacts with pass@$k$-style evaluation~\citep{yue2025doesreinforcementlearningreally,wen2025reinforcementlearningverifiablerewards}.
Complementary work improves optimization stability via refined reward assignment or denser intermediate learning signals~\citep{kazemnejad2025vinepporefiningcreditassignment}.
In contrast, our work explicitly targets sequential retries with dependent attempts and directly optimizes a Verification@K objective with attempt-level weighting to better match the long-CoT setting.

\textbf{Multi-attempt and self-feedback for revision.}
Multi-attempt prompting and retry-based mechanisms can elicit iterative reasoning and answer revision in LLMs~\citep{liu2025simpletryagainelicit,zelikman2022starbootstrappingreasoningreasoning}.
Self-feedback and tool-assisted critique loops further enable iterative refinement and reflection without requiring per-instance human corrections~\citep{madaan2023selfrefineiterativerefinementselffeedback,gou2024criticlargelanguagemodels,shinn2023reflexionlanguageagentsverbal}.
At the same time, several works characterize when self-correction succeeds or fails, emphasizing the importance of reliable external feedback~\citep{huang2024largelanguagemodelsselfcorrect}.

\textbf{Multi-turn RL and fine-grained reward assignment.}
A line of work studies RL for long-horizon, multi-turn LLM agents, where sparse terminal rewards make reward assignment difficult. In search and tool-use settings, recent methods add turn-level verifiable or judge-based rewards, intrinsic information-gain rewards, or finer step-level advantages built from grouped trajectories, shared states, or trajectory graphs~\citep{wei2025reinforcingmultiturnreasoningllm,wang2026informationgainbasedpolicyoptimization,feng2025groupingrouppolicyoptimizationllm,li2025saltstepleveladvantageassignment}. In code and tool-integrated reasoning, related approaches leverage execution feedback, learned verifier scores, single-step rewards, discounted turn returns, reward shaping, or feedback-conditioned rollout trees to support iterative self-correction~\citep{jain2025multiturncodegenerationsinglestep,ding2025empoweringmultiturntoolintegratedreasoning,ekbote2026murphymultiturngrpoself}.
Unlike these works, which mostly improve generic long-horizon agent training with task-specific turn- or step-level rewards, we study repeated attempts on the same input under hard verifier feedback and derive attempt-level weighting specifically for optimizing Verification@K.

\textbf{Verification@K, best-of-$N$, and verifier-guided test-time scaling.}
Verifier-guided selection and reranking, such as Best-of-$N$ and Pass@$K$, along with training objectives tailored to optimize these criteria, form a common approach for improving reasoning reliability~\citep{cobbe2021trainingverifierssolvemath,zhang2025generativeverifiersrewardmodeling,thrampoulidis2025advantageshapingsurrogatereward,mahdavi2025beyond}. Process-level supervision and step-wise verifiers enable richer guidance signals for search and reasoning-time allocation~\citep{lightman2023letsverifystepstep,khalifa2025processrewardmodelsthink,park2024ensemblinglargelanguagemodels}.
More broadly, work on test-time compute studies how to allocate sampling, search, and deliberation budgets and how to adapt models at test time for improved accuracy--compute tradeoffs~\citep{snell2024scalingllmtesttimecompute,wang2023selfconsistencyimproveschainthought,yao2023treethoughtsdeliberateproblem, gozeten2025testtime}.

%% file: sections/conclusion.tex
\section{Conclusion}\label{sec:conclusion}

In this paper, we introduced a formal view of long chain-of-thought as a \emph{multi-attempt} process with \emph{dependent} samples under verifier feedback, and studied direct optimization of the resulting Verification@K objective. From the corresponding policy-gradient structure, we derived attempt-aware credit assignment and proposed \emph{Calibrated Attempt-Level algorithm} (CAL), which reweights per-attempt learning signals to better match each attempt’s marginal contribution to Ver@K while keeping updates well-scaled.

Across a controlled Markov chain benchmark (including a trap variant) and real reasoning tasks (MATH and text-based Maze navigation), our empirical results show a consistent pattern: trajectory-level training (TL) entangles first-attempt correctness with later correction behavior, while naive attempt-level training (NAL) yields cleaner optimization and faster improvements in early-attempt success and efficiency in number of attempts. CAL further improves final Ver@K by sustaining useful learning pressure on later attempts when they remain the primary driver of additional gains, revealing an accuracy--compute trade-off that can be tuned depending on whether peak Verification@K or fewer retries is the priority.

A natural next step is to extend calibrated attempt-level weighting beyond hard binary verifiers to noisy/learned verifiers feedback, and to characterize when unbiasedness and variance-reduction guarantees still hold. It would also be valuable to scale to larger $K$ and study adaptive compute allocation (adaptive $K$ values) rather than using a fixed number of attempts.

%% file: sections/appendix/experimental_details.tex
\section{Experimental Details}\label{app:exp_details}

\subsection{Implementation of Ver@K GRPO Training}\label{app:verk_protocol}
We train policies with Ver@K using a multi-turn interaction loop implemented on top of the veRL/HybridFlow framework~\cite{Sheng_2025,verl_repo} and its multi-turn interaction system with an SGLang rollout backend~\cite{sglang_repo}. We modify the veRL source code to support (i) multi-attempt rollout with the same maximum output length and (ii) attempt-wise GRPO reward normalization and reweighting for the estimators described in \Cref{sect:main_results,sec:experiments}.

Each rollout allows up to $K$ assistant attempts. After attempt $k$, a deterministic verifier returns a binary reward signal $r_k\in\{0,1\}$. If $r_k=1$, the interaction terminates. Otherwise, if attempts remain, we inject a fixed feedback message indicating the attempt was incorrect and request another attempt. Attempts not reached due to early stopping are excluded from attempt-wise normalization at that step.

\subsection{Datasets and Prompting}\label{app:datasets}

\textbf{MATH.} We train and evaluate on the MATH dataset~\cite{hendrycks2021measuringmathematicalproblemsolving}. Each example is converted into a single-attempt chat prompt consisting of the problem statement followed by the instruction:
\emph{``Let’s think step by step and put the final answer in $\backslash$boxed\{\}.''}

\textbf{Maze.} We generate maze navigation datasets using the \texttt{maze-dataset} library~\cite{ivanitskiy2023configurablelibrarygeneratingmanipulating}. For each maze size ($5{\times}5$, $7{\times}7$, $9{\times}9$), we generate 10k training and 1k test instances. Mazes are sampled with the library's randomized depth-first-search lattice generator and converted to a text representation for prompting~\cite{ivanitskiy2023configurablelibrarygeneratingmanipulating}. We follow prior work studying maze navigation and multi-sample objectives for reasoning models~\cite{dao2025alphamazeenhancinglargelanguage,chen2025passktrainingadaptivelybalancing}. Each maze prompt provides an ASCII grid with walls/free cells and explicit start/goal markers. The model is instructed to output a move sequence inside \texttt{<answer>...</answer>} using the alphabet \{\texttt{U,D,L,R}\}.

\textbf{Markov chain.}
We construct a synthetic planning benchmark based on a structured Markov chain, designed to isolate multi-attempt credit assignment under a deterministic verifier with controllable difficulty. The state space is organized into $n_{\text{hubs}}$ ordered hubs with $m$ states per hub ($n_{\text{states}} = n_{\text{hubs}}\cdot m$). The transition graph is fixed across instances: transitions are predominantly local within a hub (e.g., to nearby states within $\pm 1$ positions), with sparse \emph{boundary} transitions connecting consecutive hubs; the final state in the last hub is an absorbing terminal state. Each instance samples a start state $s_0$ from the set of states that can reach the terminal. In each attempt, the model generates a token sequence representing a candidate state trajectory, and the verifier simulates it on the transition graph and returns $r_t{=}1$ iff the trajectory reaches the terminal, uses only valid transitions, never decreases the state id, and has length equal to the precomputed shortest-path distance from $s_0$ to the terminal (i.e., it is an optimal-length path). We train with Ver@K by separating attempts with a special \texttt{RETRY} token; after a failed attempt, the interaction resets to the same start state and the model retries up to $K$ attempts with early stopping. We also consider a \emph{trap} variant where a trap state is sampled per rollout from two pre-terminal states in the last hub; visiting the trap causes immediate failure, and shortest-path distances are computed on the transition graph with the trap removed.

\subsection{Verifiers and Rewards}\label{app:verifiers}
\textbf{MATH verifier.} We extract the final boxed expression from the model output, apply normalization, and check equivalence against the ground-truth answer. The per-attempt reward is $r_k{=}1$ if correct and $0$ otherwise.

\textbf{Maze verifier.} We extract the move string from \texttt{<answer>} tags, normalize to \{\texttt{U,D,L,R}\}, and simulate the trajectory on the maze. The per-attempt reward is $r_k{=}1$ iff the move sequence reaches the goal without crossing walls or leaving the grid; otherwise $0$.

\textbf{Markov chain verifier.}
Given a proposed state sequence for an attempt, we simulate it on the chain and return $r_k\in\{0,1\}$.
An attempt receives $r_k=1$ iff it reaches the absorbing terminal state, all transitions are valid under the chain's transition structure,
the state id never decreases, and the attempt length equals the precomputed shortest-path distance from the start state.
In the trap variant, visiting the sampled trap state causes immediate failure and shortest-path distances are computed with the trap removed.

\subsection{Rollout, Sampling, and Optimization}\label{app:training_setup}

\textbf{Model and algorithm.} For MATH and Maze, we fine-tune Qwen3-4B-Instruct-2507~\cite{yang2025qwen3technicalreport} using GRPO~\cite{shao2024deepseekmathpushinglimitsmathematical},
implemented in veRL~\cite{Sheng_2025,verl_repo} with our multi-turn and estimator modifications (Appendix~\ref{app:verk_estimators}).
For the Markov chain task, we initialize the policy from an SFT checkpoint and then apply the same Ver@K GRPO procedure.

\textbf{Markov-chain SFT initialization.}
The SFT checkpoint used to initialize the Markov-chain runs is trained on synthetic optimal state trajectories. For each synthetic instance, we sample a start state from the states that can reach the terminal; in the trap setting, we additionally sample the trap state from the two pre-terminal states in the last hub. We then compute the corresponding shortest valid trajectory to the absorbing terminal state on the known transition graph (with the trap removed when applicable) and train the model to emit that state-id sequence. This SFT checkpoint is used only as initialization, and all TL-GRPO, NAL-GRPO, and CAL-GRPO runs then use the same Ver@K GRPO setup described below.

\textbf{Rollout and decoding.} We use SGLang for rollout generation~\cite{sglang_repo}. Decoding uses temperature 0.6, top-$p$ 0.95, and top-$k$ 20. The GRPO group size (number of rollouts per prompt) is $N{=}20$ for MATH and $N{=}16$ for Maze. We evaluate every 5 iterations. Unless otherwise stated, reported curves correspond to a single training run. For the $5{\times}5$ and $9{\times}9$ Maze Ver@2 experiments, we repeat training with five random seeds and report the mean together with 95\% confidence intervals across runs. We focus this multi-seed evaluation on these Maze Ver@2 settings due to computational constraints.

For Markov chain setting, we generate rollouts with a lightweight token-level simulator (no tool calls): the model emits discrete state-id tokens in $\{0,\dots,n_{\text{states}}{-}1\}$, with attempts separated by an injected RETRY token that resets the attempt to the same start state. Concretely, the RETRY token id is set to $n_{\text{states}}$ and the PAD token id to $n_{\text{states}}{+}1$, where $n_{\text{states}} = n_{\text{hubs}}\cdot m$. We sample with temperature $1.0$ and disable top-$k$ truncation (top-$k{=}0$). For each update, we sample a batch of start states, repeat each start state to form a rollout group, and roll out $K$ attempts per start. In the trap variant, start states are sampled only from hubs before the last hub so that the sampled trap remains ahead in the chain.

\textbf{Sequence lengths.} We cap each attempt to 512 output tokens for MATH and for the $5{\times}5$ Maze task, and to 1024 new tokens for the $9{\times}9$ Maze task. With $K{=}2$, the total assistant budget is 1024 tokens (MATH and $5{\times}5$ Maze) or 2048 tokens ($9{\times}9$ Maze) plus a small buffer for system/feedback, yielding a maximum rollout length of 1280 tokens (MATH and $5{\times}5$ Maze) or 2304 tokens ($9{\times}9$ Maze).

For the Markov chain setting, we cap the length of each attempt to a fixed token budget. By default this budget is set to the total number of states, but it is never allowed to exceed the model’s maximum context length. We choose the per-attempt cap so that all attempts can be packed into a single context window, while also reserving space for the inserted \texttt{RETRY} separator tokens between attempts. An attempt terminates if it reaches the absorbing terminal state, emits \texttt{RETRY}/\texttt{PAD}/EOS, outputs an out-of-range state id (invalid), or hits the per-attempt token cap.

\textbf{Optimization.} We train with AdamW at learning rate $10^{-6}$ with a KL penalty coefficient of 0.001, a clip ratio $\epsilon=0.2$, and no entropy bonus. We use a train batch size of 256, PPO mini-batch size 64, and PPO micro-batch size 8 per GPU with dynamic batching. We train for 2 epochs on MATH and 1 epoch on Maze. We used 2$\times$ NVIDIA L40S or A100 GPUs during the training.

For Markov chain setting, we use AdamW with parameters $(0.9,0.95)$, PPO-style ratio clipping with $\epsilon=0.2$. We use learning rate $10^{-5}$ and a KL penalty coefficient of $0.02$. For non-trap runs: We train for 1600 steps, batch size of 8, rollout group size $N=40$, and $K=4$ attempts per start state. Trap runs: We train for 5000 steps, batch size of 16, rollout group size $N=40$, and $K=2$ attempts per start state. We evaluate using $500\times n_{\text{states}}$ start states and evaluate every 40 iterations.

\begin{tcolorbox}[title=Verification feedback in MATH, colback=white, colframe=black, breakable]
\ttfamily\footnotesize
Your previous attempt(s) were incorrect. Review them and try again.\\
Solve the problem again from scratch, correcting any mistakes.\\
Final answer must be \verb|\boxed{<answer>}|.
\end{tcolorbox}

\begin{tcolorbox}[title=Verification feedback in Maze, colback=white, colframe=black, breakable]
\ttfamily\footnotesize
Your previous attempt(s) were incorrect. Review them and try again.\\
Solve the maze again from scratch, correcting any mistakes.\\
Final answer must be a move string over U,D,L,R only in \verb|<answer>...</answer>|.
\end{tcolorbox}

\begin{tcolorbox}[title=Example Answer, colback=white, colframe=black]
\centering
\ttfamily\footnotesize
\texttt{<answer>UU</answer>}
\end{tcolorbox}

\begin{tcolorbox}[title=Example Maze Question, colback=white, colframe=black, breakable]
\ttfamily\footnotesize
You need to solve the following maze.\\

'*' denotes the wall that you cannot walk through, '.' denotes available area that you can walk through. 'S' denotes the starting point, 'E' denotes the destination.\\

You need to start from the starting point and cross through the available area to reach the destination. There are four movement actions, including Left, Right, Up, Down.\\

Use L to denote Left movement, R to denote Right movement, U to denote Up movement, and D to denote Down movement.\\

You can analyze the maze to find the correct path, and you should write the final path in the <answer></answer>, e.g., <answer>LLRRDUL</answer>.\\

\#\# Maze

{\ttfamily\obeylines\obeyspaces\noindent
*********
*.....*E*
*****.*.*
*...*..S*
*.*****.*
*.......*
*.*******
*.......*
*********
}

Now try to analyze the maze and put the final path in the <answer></answer>.
\end{tcolorbox}

\begin{tcolorbox}[title=Markov Chain reward and termination logic, colback=white, colframe=black, breakable]
\ttfamily\footnotesize
For each attempt, the model autoregressively emits state-id tokens.\\
Reward=1 iff:
(1) the terminal absorbing state is reached; (2) all transitions follow the allowed adjacency; (3) the state id never decreases; (4) the number of steps equals the shortest-path distance from the start.\\
Trap variant: a trap state is sampled per rollout from two pre-terminal states in the last hub; visiting it makes Reward=0, and
shortest paths are computed with the trap removed.
\end{tcolorbox}

\begin{tcolorbox}[title=Example Trajectories for Markov Chain for Ver@2, colback=white, colframe=black, breakable]
\ttfamily\footnotesize
Example parameters: $n_{\text{hubs}}=5$, $m=6$ ($n_{\text{states}}=30$), start $s_0=0$, terminal=29.\\[2pt]
Attempt 1 (failure, backtrack):
\quad 0 $\rightarrow$ 2 $\rightarrow$ 0 $\rightarrow$ 1 $\rightarrow$ \dots\\
\texttt{RETRY}\\
Attempt 2 (success, shortest path):
\quad 0 $\rightarrow$ 1 $\rightarrow$ 3 $\rightarrow$ 5 $\rightarrow$ \dots $\rightarrow$ 27 $\rightarrow$ 29.\\[2pt]
Trap variant (e.g., trap=28): the above success remains valid; any visit to 28 during an attempt yields $r_k=0$.
\end{tcolorbox}

\subsection{GRPO Variants}\label{app:verk_estimators}

Fix a prompt instance $x$ (and its ground-truth $y$), and draw a GRPO group of $N$ i.i.d. multi-attempt
trajectories $\tau_1,\dots,\tau_N \sim p_\theta(\cdot\mid x)$, as in \Cref{sect:main_results}.
Each trajectory $\tau_n$ executes attempts $k\in\{1,\dots,K\}$ with early stopping: at attempt $k$ the
policy produces an output $y_{n,k}\sim\pi_\theta(\cdot\mid s_{n,k-1})$ and a deterministic verifier returns
$r_{n,k}\in\{0,1\}$. The stopping time is
\[
T_n := \min\{k\le K: r_{n,k}=1\}\quad \text{(or $T_n=K$ if no success).}
\]
Define the reach indicator $\mathbf{1}[T_n\ge k]$ and the reached-set
\[
S_k := \{n\in[N]: T_n\ge k\}.
\]
We also define a token mask $m_{n,t}\in\{0,1\}$ indicating model-generated tokens,
and let $\mathrm{turn}(n,t)\in\{1,\dots,K\}$ denote the attempt index of token $t$. We use the standard GRPO normalization for all variants below, by dividing centered rewards by the corresponding within-group empirical standard deviation. This rescaling is not part of the RLOO derivations in Sections~3 and~5, but we use it in practice because it improves optimization, and \Cref{fig:maze_ver2_stdnorm} reports a direct ablation for this on Maze Ver@2.

\textbf{Trajectory-Level GRPO (TL-GRPO).}
Let the trajectory-level Ver@K reward be
\[
r^{\mathrm{TL}}_n := r_{n,T_n} = \mathbf{1}\Big[\max_{k\le K} r_{n,k}=1\Big].
\]
We compute the usual GRPO group-normalized advantage
\[
a^{\mathrm{TL}}_n = \frac{r^{\mathrm{TL}}_n-\mu}{\sigma},
\qquad
\mu=\frac{1}{N}\sum_{n=1}^N r^{\mathrm{TL}}_n,\quad
\sigma=\mathrm{Std}\big(\{r^{\mathrm{TL}}_n\}_{n=1}^N\big),
\]
and assign token-level advantages
\[
A_{n,t} = a^{\mathrm{TL}}_n \cdot m_{n,t}.
\]

\textbf{Naive Attempt-Level GRPO (NAL-GRPO).}
For each attempt $k$, we normalize the per-attempt verifier outcomes $r_{n,k}$ \emph{only over trajectories that
reach attempt $k$}:
\[
a^{\mathrm{NAL}}_{n,k}
=
\frac{r_{n,k}-\mu_k}{\sigma_k}
\quad\text{for } n\in S_k,
\qquad
\mu_k=\frac{1}{|S_k|}\sum_{n\in S_k} r_{n,k},\quad
\sigma_k=\mathrm{Std}\big(\{r_{n,k}\}_{n\in S_k}\big).
\]
Each token inherits the attempt-level advantage of its attempt:
\[
A_{n,t}
=
a^{\mathrm{NAL}}_{n,\mathrm{turn}(n,t)} \cdot m_{n,t}.
\]

\textbf{Calibrated Attempt-Level GRPO (CAL-GRPO).}
CAL-GRPO starts from NAL-GRPO and reweights attempts by a future-only factor that approximates each attempt’s
marginal contribution to final Ver@K success. Define the empirical per-attempt success rates
\[
\hat p_k := \frac{1}{|S_k|}\sum_{n\in S_k} r_{n,k}.
\]
Then define future-only weights:
\[
w_k = \prod_{j=k+1}^{K} (1-\hat p_j),
\qquad
w_K = 1.
\]
We renormalize weights within the group to keep update scales comparable, e.g.,
\[
\bar w = \frac{1}{|\mathcal{K}|}\sum_{k\in \mathcal{K}} w_k,
\quad \mathcal{K}:=\{k\in[K]: |S_k|>0\},
\qquad
\tilde w_k = \frac{w_k}{\bar w}.
\]
Final token-level advantages are
\[
A_{n,t}
=
\tilde w_{\mathrm{turn}(n,t)}
a^{\mathrm{NAL}}_{n,\mathrm{turn}(n,t)}
m_{n,t}.
\]

%% file: sections/appendix/additional_details.tex
\section{Additional Experimental Results for \Cref{sec:experiments}}\label{app:add_results}

\begin{figure*}[t!]
    \centering
    \begin{subfigure}[t]{0.32\textwidth}
        \centering
        \includegraphics[width=\linewidth]{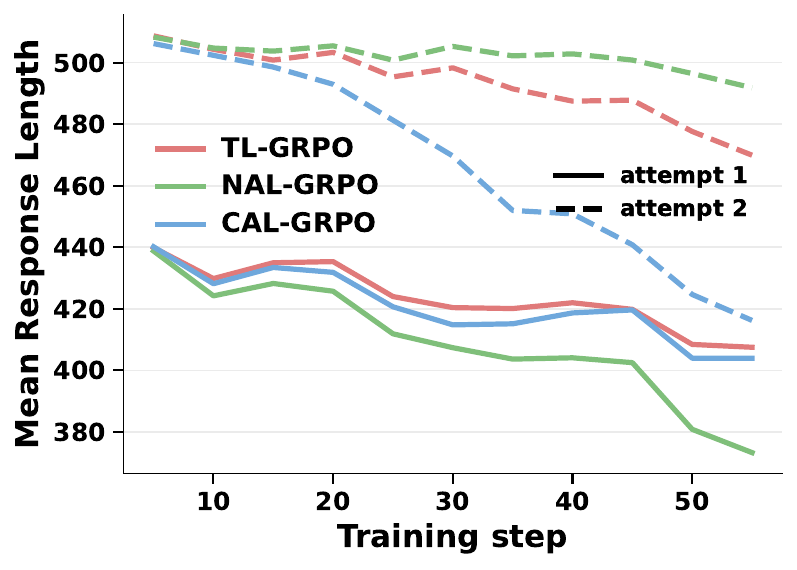}
        \caption{\textbf{Mean response length} per attempt vs.\ training step (baselines labeled in legend: TL-GRPO, NAL-GRPO, CAL-GRPO).}
        \label{fig:app_math_mean_len}
    \end{subfigure}
    \hfill
    \begin{subfigure}[t]{0.32\textwidth}
        \centering
        \includegraphics[width=\linewidth]{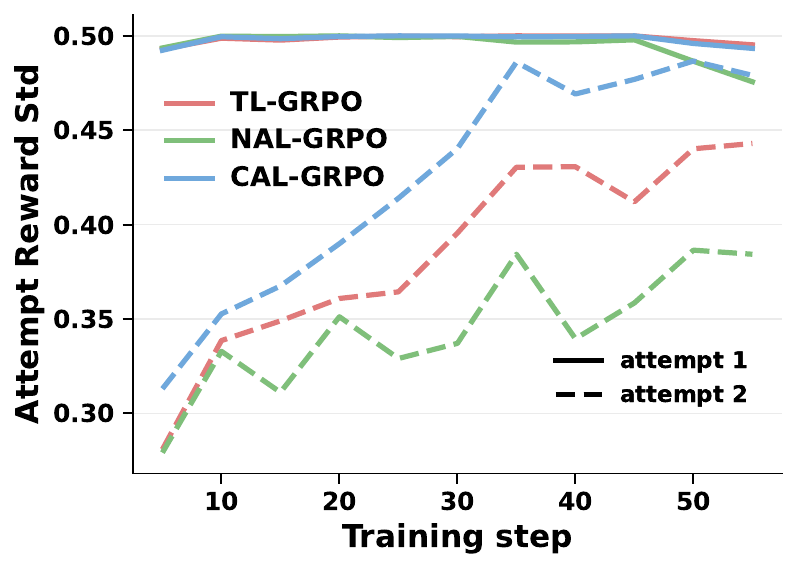}
        \caption{\textbf{Attempt reward standard deviation} within each GRPO group (signal magnitude for normalization), shown separately for attempts 1 and 2.}
        \label{fig:app_math_reward_std}
    \end{subfigure}
    \hfill
    \begin{subfigure}[t]{0.32\textwidth}
        \centering
        \includegraphics[width=\linewidth]{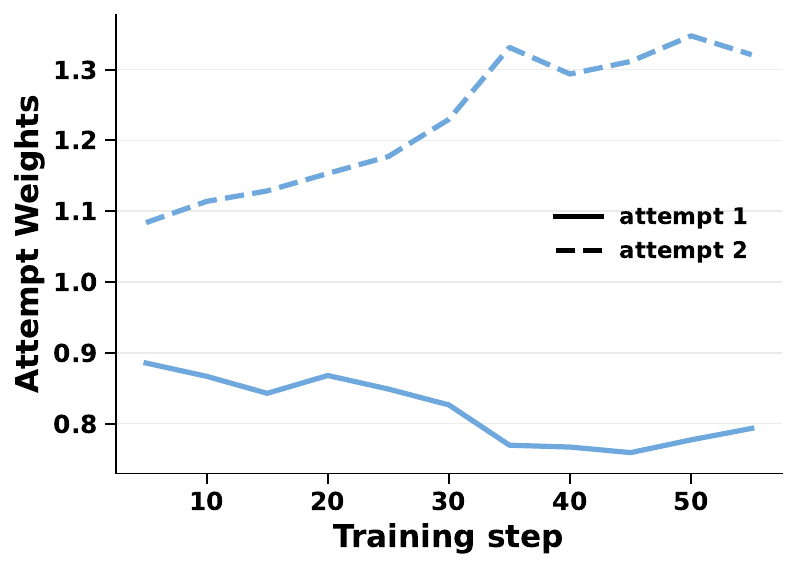}
        \caption{\textbf{Attempt weights} used by CAL-GRPO (future-only reweighting), plotted for attempts 1 and 2 over training.}
        \label{fig:app_math_weights}
    \end{subfigure}

    \caption{\textbf{Additional metrics for Ver@2 training on MATH.} (a) Response length dynamics per attempt. (b) Per-attempt reward variability used by normalization. (c) Learned future-only attempt reweighting in CAL-GRPO.}
    \label{fig:app_math_diagnostics}
\end{figure*}

\textbf{Discussion on the additional results.}
\Cref{fig:app_math_diagnostics,fig:app_maze_diagnostics,fig:app_maze9_diagnostics,fig:app_maze9_ver4_diagnostics,fig:app_maze7_ver4_diagnostics,fig:maze_ver2_groupsize}
report additional diagnostics and ablations for the three estimators (TL-GRPO/NAL-GRPO/CAL-GRPO). For Maze, we include Ver@2 on $5{\times}5$ and $9{\times}9$ (see also \Cref{fig:maze9_ver2_acc,fig:maze9_ver2_estimators}),
as well as Ver@4 on $7{\times}7$ and $9{\times}9$ (see \Cref{fig:maze7_ver4_estimators,fig:maze9_ver4_estimators}),
and rollout-group-size and standard-deviation-normalization ablations under Ver@2 (see \Cref{fig:maze_ver2_groupsize}). We make the following observations:

\textbf{Weights track marginal contribution.}
Across tasks, the future-dependent weights in CAL-GRPO shift learning signal away from early attempts and toward later attempts as training progresses.
For Ver@2, CAL-GRPO downweights attempt 1 and upweights attempt 2 over training
(\Cref{fig:app_math_weights,fig:app_maze_weights,fig:app_maze9_weights}).
For Ver@4, this pattern generalizes: attempts 1--2 are increasingly downweighted while attempts 3--4 are upweighted
(\Cref{fig:app_maze9_ver4_weights,fig:app_maze7_ver4_weights}),
which counteracts the fact that later attempts are reached on a shrinking subset of hard failures and would otherwise receive a vanishing share of updates.
The reweighting is strongest in easier/near-saturating regimes---notably $5{\times}5$ under Ver@2 and $7{\times}7$ under Ver@4---where retries become highly reliable and the marginal value of further improving the first attempt diminishes.
On $9{\times}9$ (both Ver@2 and Ver@4), weights remain less extreme, consistent with a harder regime where multiple attempts continue to contribute meaningfully to final Ver@K accuracy (\Cref{fig:maze9_ver2_acc,fig:maze9_ver4_acc}).

\textbf{Gains without longer generations.}
Mean response length decreases over training on MATH and across all Maze settings we report, for both Ver@2 and Ver@4
(\Cref{fig:app_math_mean_len,fig:app_maze_mean_len,fig:app_maze9_mean_len,fig:app_maze9_ver4_mean_len,fig:app_maze7_ver4_mean_len}),
indicating that Ver@K gains are not driven by longer generations.
On $5{\times}5$ (Ver@2), CAL-GRPO notably shortens the second attempt late in training, directly reducing test-time compute under early stopping.
On $9{\times}9$ (Ver@2), the dominant effect is a steady reduction in attempt-1 length with controlled retry length, suggesting improved first-pass planning without incurring longer corrections.
For Ver@4 on $7{\times}7$ and $9{\times}9$, response lengths generally drop across attempts 1-4, with later attempts often shortening substantially late in training (\Cref{fig:app_maze7_ver4_mean_len,fig:app_maze9_ver4_mean_len}), consistent with learning more efficient correction rather than longer retries.

\textbf{Attempt-dependent normalization matters.}
The within-group reward standard deviation differs across attempts and evolves over training for both Ver@2 and Ver@4
(\Cref{fig:app_math_reward_std,fig:app_maze_reward_std,fig:app_maze9_reward_std,fig:app_maze9_ver4_reward_std,fig:app_maze7_ver4_reward_std}),
reinforcing the claim that a single trajectory-level normalization can mis-scale gradients for later attempts. This heterogeneity becomes more salient as $K$ grows: under Ver@4, attempts 3--4 can exhibit markedly different reward-variance dynamics than attempts 1--2. On $5{\times}5$ with Ver@2, attempt-2 variability rises early but collapses late as retry outcomes become near-deterministic once performance approaches ceiling;
by contrast, on $9{\times}9$ with Ver@2 the retry variance remains substantial throughout training, supporting the fact that Ver@2 does not fully saturate and learning signal on retries remains informative (\Cref{fig:maze9_ver2_acc}).
Similarly, for Ver@4, variance patterns differ across $7{\times}7$ vs.\ $9{\times}9$ (\Cref{fig:app_maze7_ver4_reward_std,fig:app_maze9_ver4_reward_std}), aligning with the corresponding Ver@4 learning curves (\Cref{fig:maze7_ver4_acc,fig:maze9_ver4_acc}).
Overall, these demonstrations help explain why CAL-GRPO’s reweighting becomes increasingly important for larger $K$ as it preserves learning pressure on later-attempt correction when those retries remain a key driver of marginal Ver@K improvements.

\textbf{Ablation on standard deviation normalization.}
Consistent with the attempt-dependent variance patterns above, the direct Maze Ver@2 ablation in \Cref{fig:maze_ver2_stdnorm} shows that dividing by the within-group empirical standard deviation improves sample efficiency for TL-GRPO, NAL-GRPO, and CAL-GRPO, with the clearest gain for CAL-GRPO. Although our theoretical derivations are written in an RLOO-style form without this GRPO rescaling, this empirical result motivates our use of the standard GRPO normalization in all experiments.

\textbf{Ablation on group size.}
CAL-GRPO estimates its weights directly from the current rollout group. For Ver@2, the quantity is
$\hat p_2 = \frac{1}{|S_2|}\sum_{n\in S_2} r_{n,2}$,
which is estimated from the trajectories in the current group that reach attempt 2, and the calibrated weights reduce to $w_1 = 1-\hat p_2$ and $w_2 = 1$. Hence, the quality of the calibration depends on how well the current rollouts estimate second attempt success.
\Cref{fig:maze_ver2_groupsize_acc,fig:maze_ver2_groupsize_reached} compares rollout group sizes $N=10$ and $N=16$ on Maze under Ver@2.
Using $N=16$ improves sample efficiency for all three methods, with the clearest gain for CAL-GRPO early in training.
Larger groups yield more trajectories that reach attempt 2, so both the per-attempt normalization and the estimate of $\hat p_2$ are based on a larger and therefore less noisy sample of current rollouts. This becomes increasingly relevant later in training, when attempt-1 success rises and the reached set $S_2$ shrinks. These results motivate our use of a moderately large rollout group size for CAL-GRPO throughout our experiments.

\begin{figure*}[t!]
    \centering
    \begin{subfigure}[t]{0.32\textwidth}
        \centering
        \includegraphics[width=\linewidth]{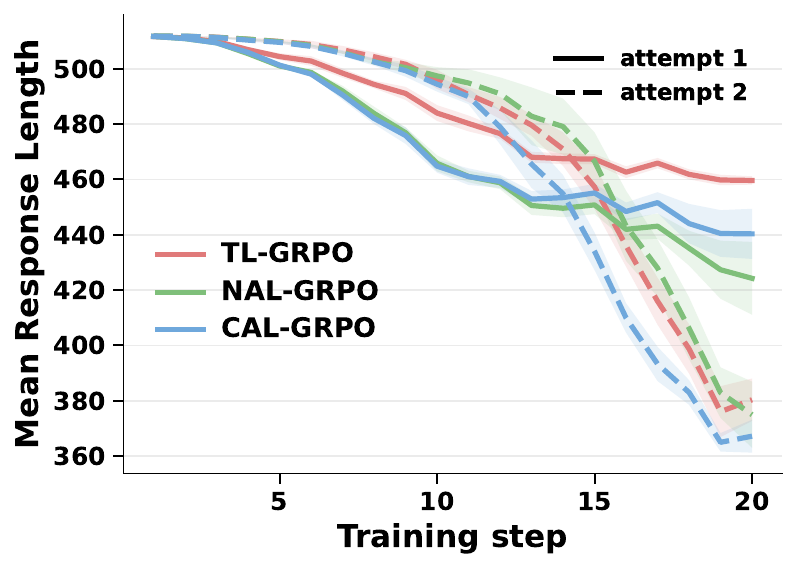}
        \caption{\textbf{Mean response length} per attempt vs.\ training step (baselines labeled in legend: TL-GRPO, NAL-GRPO, CAL-GRPO).}
        \label{fig:app_maze_mean_len}
    \end{subfigure}
    \hfill
    \begin{subfigure}[t]{0.32\textwidth}
        \centering
        \includegraphics[width=\linewidth]{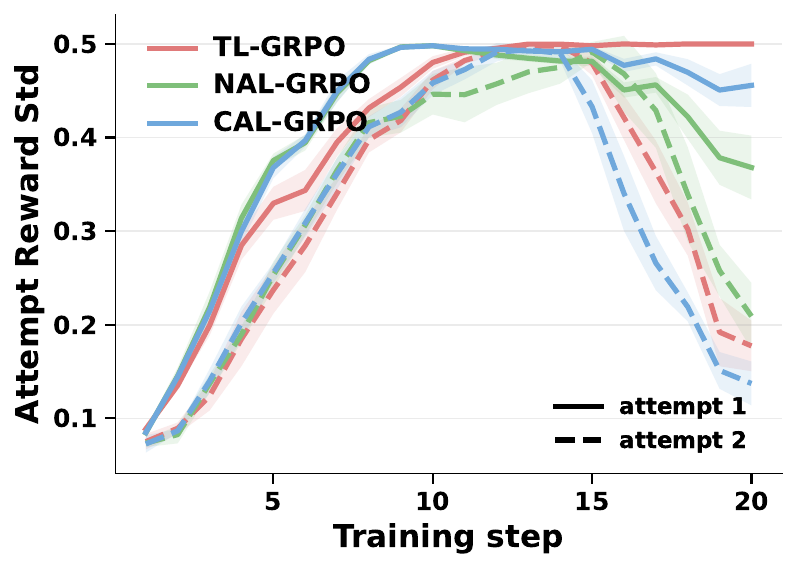}
        \caption{\textbf{Attempt reward standard deviation} within each GRPO group, shown separately for attempts 1 and 2.}
        \label{fig:app_maze_reward_std}
    \end{subfigure}
    \hfill
    \begin{subfigure}[t]{0.32\textwidth}
        \centering
        \includegraphics[width=\linewidth]{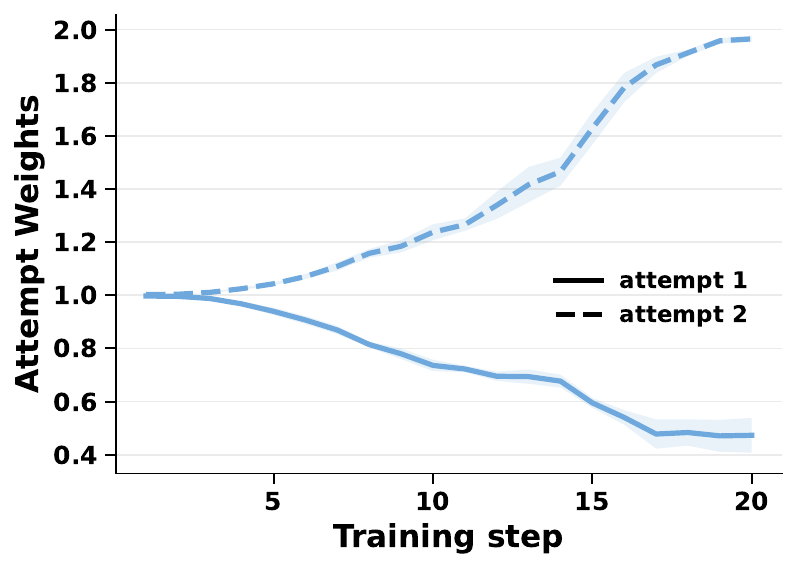}
        \caption{\textbf{Attempt weights} used by CAL-GRPO (future-only reweighting), for attempts 1 and 2 over training.}
        \label{fig:app_maze_weights}
    \end{subfigure}

    \caption{\textbf{Additional metrics for Ver@2 training on 5x5 Maze task.} (a) Response length dynamics per attempt. (b) Per-attempt reward variation used during the per attempt normalization. (c) Learned future-only attempt reweighting in CAL-GRPO.}
    \label{fig:app_maze_diagnostics}
\end{figure*}

\begin{figure*}[t!]
    \centering
    \begin{subfigure}[t]{0.32\textwidth}
        \centering
        \includegraphics[width=\linewidth]{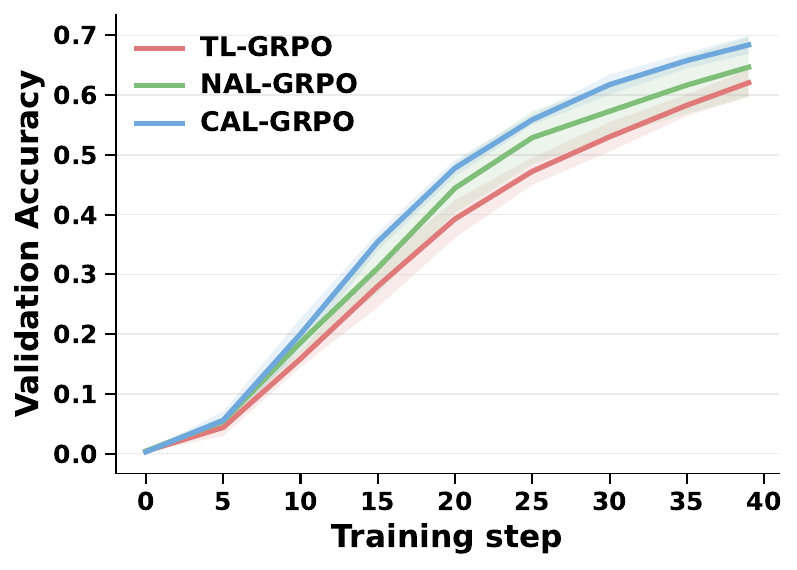}
        \caption{\textbf{Ver@2 validation accuracy on 9$\times$9 Maze} across training; CAL-GRPO achieves the highest Ver@2, with NAL-GRPO improving over TL-GRPO.}
        \label{fig:maze9_ver2_acc}
    \end{subfigure}
    \hfill
    \begin{subfigure}[t]{0.32\textwidth}
        \centering
        \includegraphics[width=\linewidth]{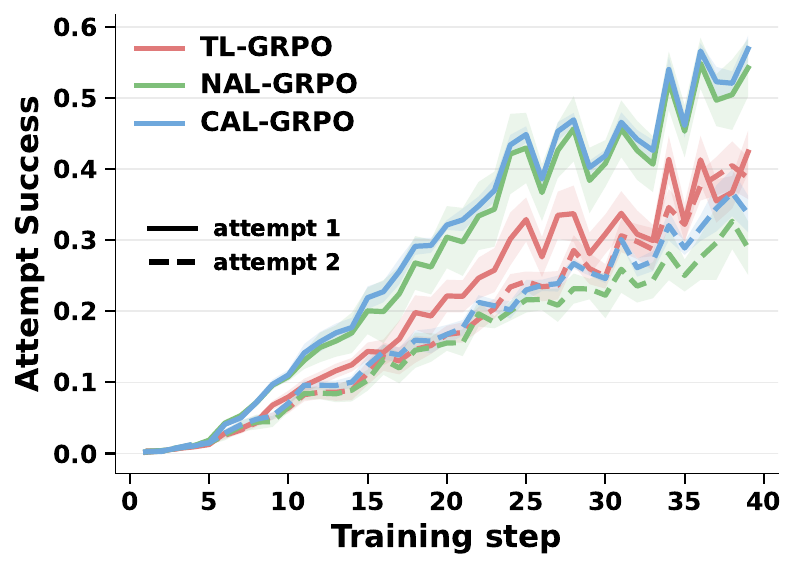}
        \caption{\textbf{Attempt-wise success.} Solid: attempt-1 success. Dashed: attempt-2 success given attempt-1 failure; NAL-GRPO boost attempt-1 success while CAL improves correction.}
        \label{fig:maze9_ver2_attempt_success}
    \end{subfigure}
    \hfill
    \begin{subfigure}[t]{0.32\textwidth}
        \centering
        \includegraphics[width=\linewidth]{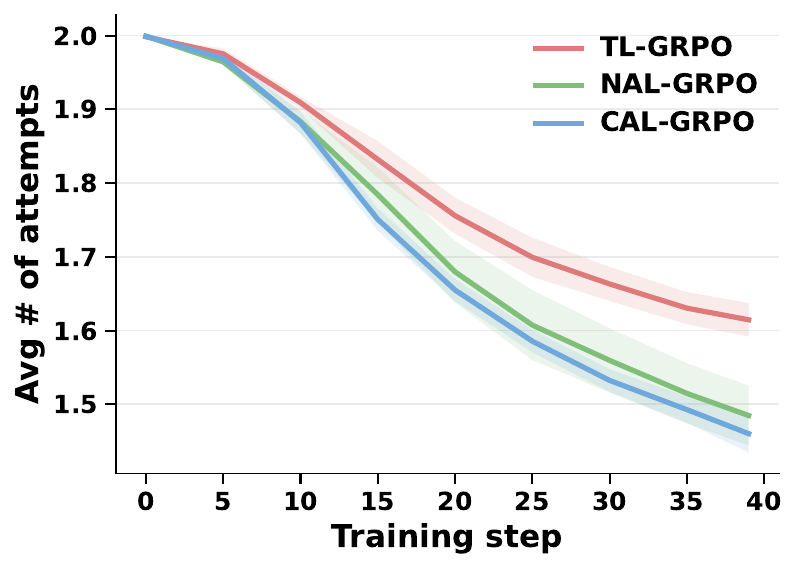}
        \caption{\textbf{Average attempts} Ver@2; NAL/CAL reduce number of attempts relative to TL-GRPO.}
        \label{fig:maze9_ver2_avg_attempts}
    \end{subfigure}

    \caption{\textbf{Optimizing the Ver@2 objective on 9$\times$9 Maze with GRPO variants.} We compare Trajectory-Level GRPO (TL-GRPO), Naive Attempt-Level GRPO (NAL-GRPO), and Calibrated Attempt-Level GRPO (CAL-GRPO). \textbf{(a)} CAL-GRPO attains the best Ver@2. \textbf{(b)} Improvements mainly come from higher attempt-1 success with competitive attempt-2 correction. \textbf{(c)} NAL/CAL reduce the average number of attempts compared to TL-GRPO.}
    \label{fig:maze9_ver2_estimators}
\end{figure*}

\begin{figure*}[t!]
    \centering
    \begin{subfigure}[t]{0.32\textwidth}
        \centering
        \includegraphics[width=\linewidth]{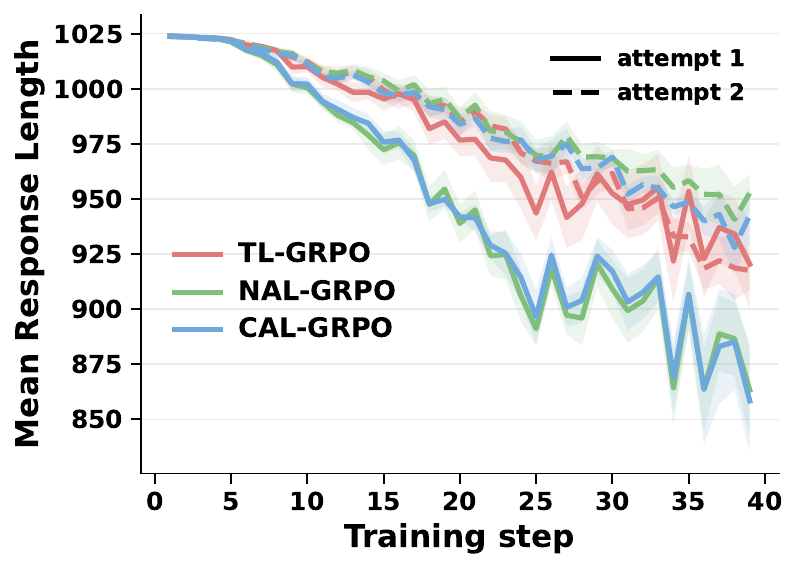}
        \caption{\textbf{Mean response length} per attempt vs.\ training step (baselines labeled in legend: TL-GRPO, NAL-GRPO, CAL-GRPO).}
        \label{fig:app_maze9_mean_len}
    \end{subfigure}
    \hfill
    \begin{subfigure}[t]{0.32\textwidth}
        \centering
        \includegraphics[width=\linewidth]{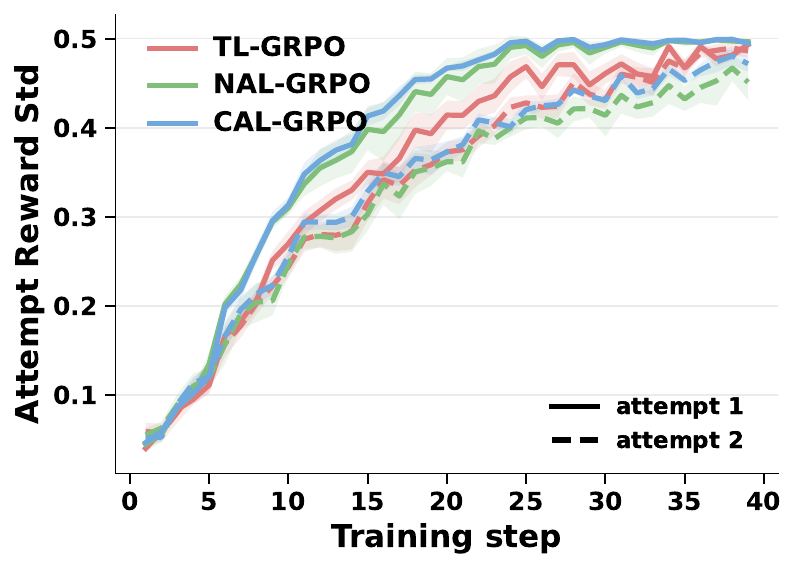}
        \caption{\textbf{Attempt reward standard deviation} within each GRPO group, shown separately for attempts 1 and 2.}
        \label{fig:app_maze9_reward_std}
    \end{subfigure}
    \hfill
    \begin{subfigure}[t]{0.32\textwidth}
        \centering
        \includegraphics[width=\linewidth]{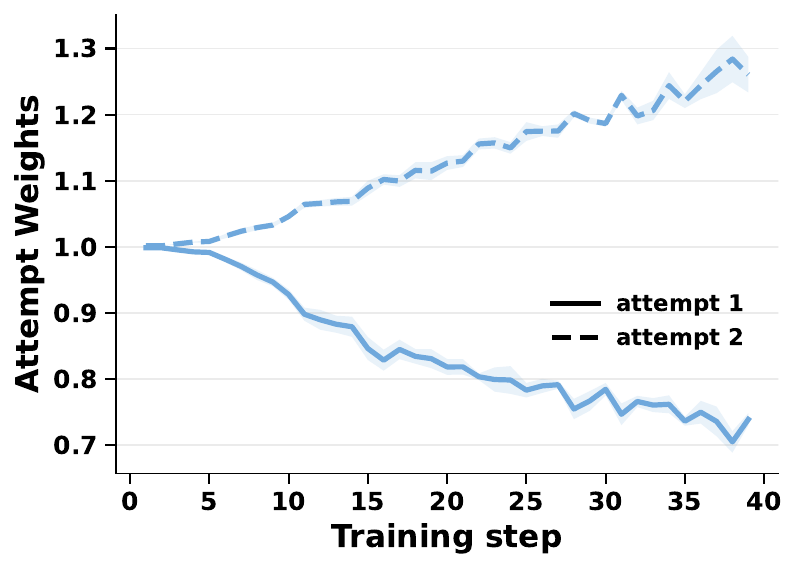}
        \caption{\textbf{Attempt weights} used by CAL-GRPO (future-only reweighting), for attempts 1 and 2 over training.}
        \label{fig:app_maze9_weights}
    \end{subfigure}

    \caption{\textbf{Additional metrics for Ver@2 training on 9$\times$9 Maze task.} (a) Response length dynamics per attempt. (b) Per-attempt reward variation used during the per-attempt normalization. (c) Learned future-only attempt reweighting in CAL-GRPO.}
    \label{fig:app_maze9_diagnostics}
\end{figure*}

\begin{figure*}[t!]
    \centering
    \begin{subfigure}[t]{0.32\textwidth}
        \centering
        \includegraphics[width=\linewidth]{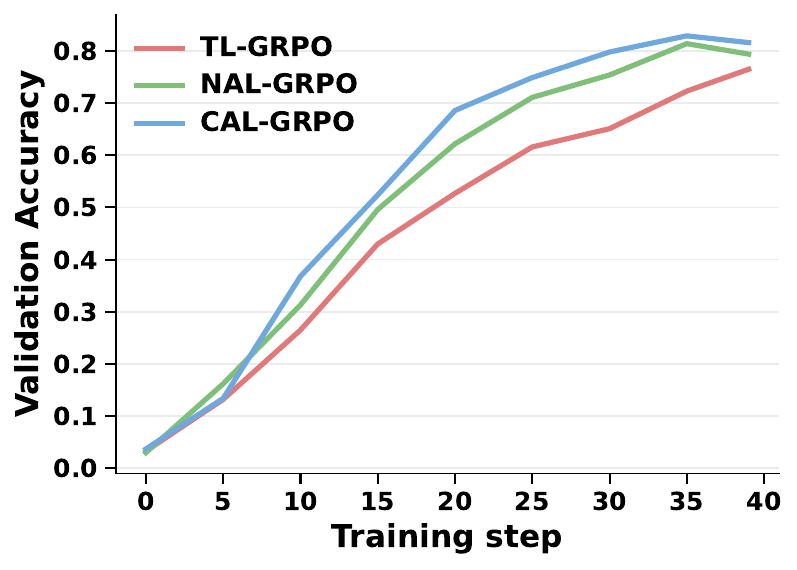}
        \caption{\textbf{Ver@4 validation accuracy on 9$\times$9 Maze} across training; CAL-GRPO achieves the highest Ver@4, with NAL-GRPO improving over TL-GRPO.}
        \label{fig:maze9_ver4_acc}
    \end{subfigure}
    \hfill
    \begin{subfigure}[t]{0.32\textwidth}
        \centering
        \includegraphics[width=\linewidth]{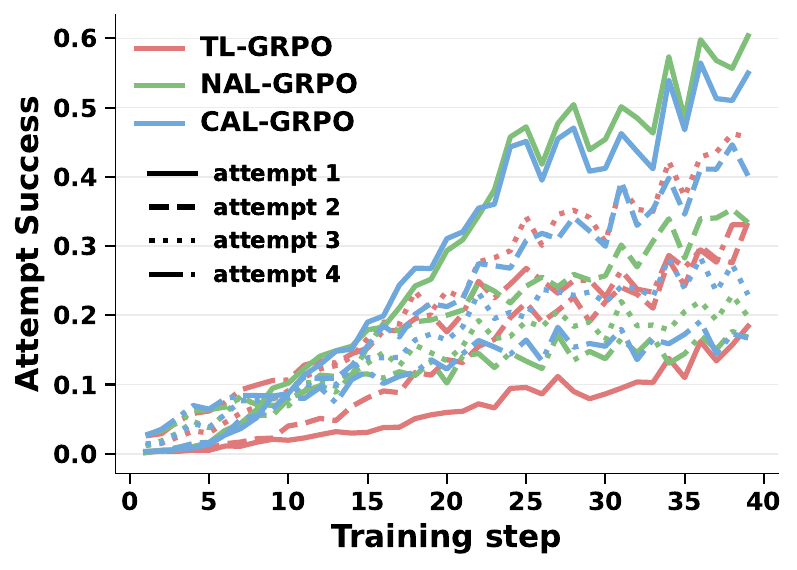}
        \caption{\textbf{Attempt-wise success.} Curves show success on attempt $i$ conditioned on failing all earlier attempts. NAL-GRPO boosts early-attempt success while CAL-GRPO improves later-attempt correction.}
        \label{fig:maze9_ver4_attempt_success}
    \end{subfigure}
    \hfill
    \begin{subfigure}[t]{0.32\textwidth}
        \centering
        \includegraphics[width=\linewidth]{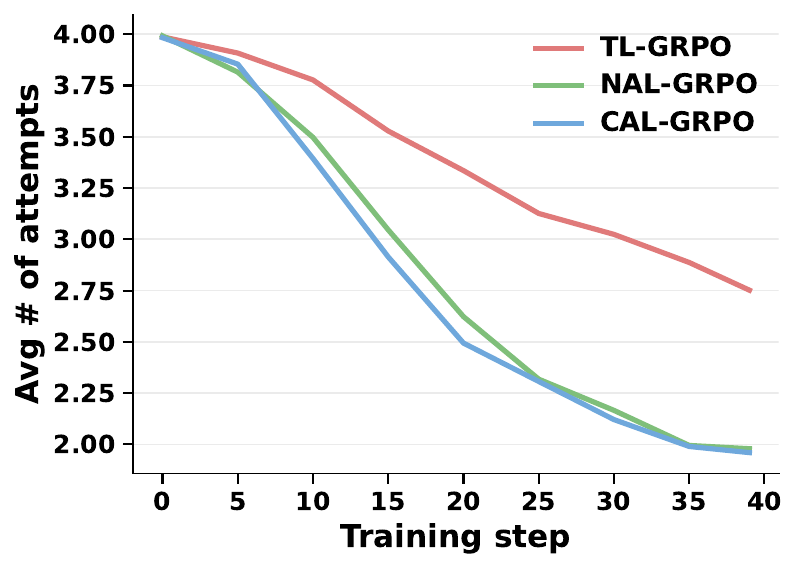}
        \caption{\textbf{Average attempts} under Ver@4; NAL/CAL reduce the number of attempts relative to TL-GRPO.}
        \label{fig:maze9_ver4_avg_attempts}
    \end{subfigure}

    \caption{\textbf{Optimizing the Ver@4 objective on 9$\times$9 Maze with GRPO variants.} We compare Trajectory-Level GRPO (TL-GRPO), Naive Attempt-Level GRPO (NAL-GRPO), and Calibrated Attempt-Level GRPO (CAL-GRPO). \textbf{(a)} CAL-GRPO attains the best Ver@4. \textbf{(b)} Improvements reflect stronger attempt-wise performance, with CAL-GRPO emphasizing correction on later attempts. \textbf{(c)} NAL/CAL reduce the average number of attempts compared to TL-GRPO.}
    \label{fig:maze9_ver4_estimators}
\end{figure*}

\begin{figure*}[t!]
    \centering
    \begin{subfigure}[t]{0.32\textwidth}
        \centering
        \includegraphics[width=\linewidth]{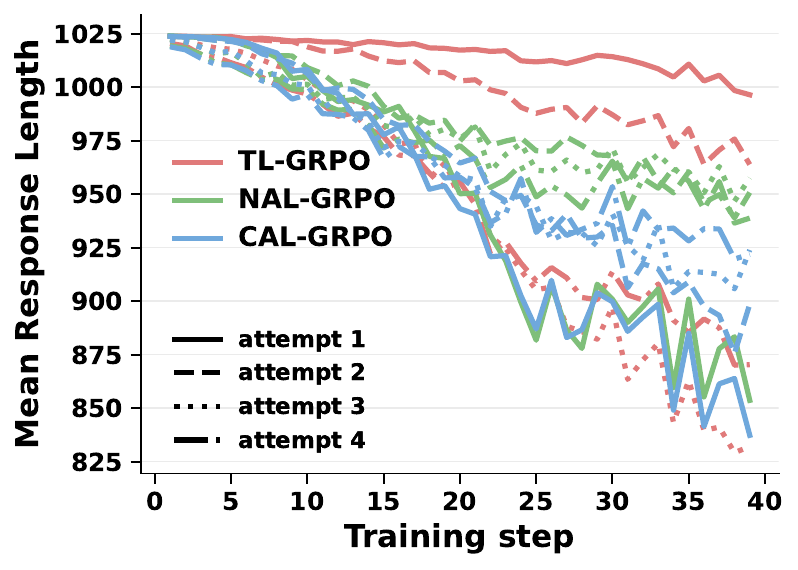}
        \caption{\textbf{Mean response length} per attempt vs.\ training step (baselines labeled in legend: TL-GRPO, NAL-GRPO, CAL-GRPO).}
        \label{fig:app_maze9_ver4_mean_len}
    \end{subfigure}
    \hfill
    \begin{subfigure}[t]{0.32\textwidth}
        \centering
        \includegraphics[width=\linewidth]{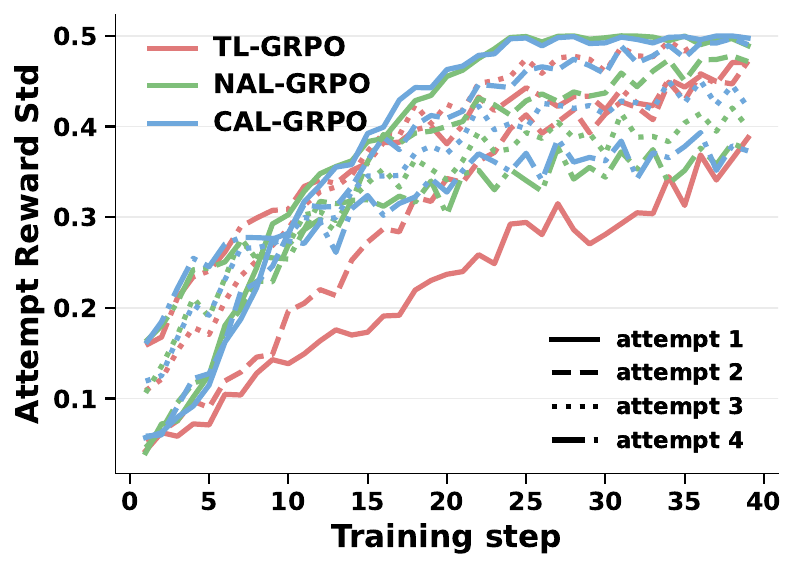}
        \caption{\textbf{Attempt reward standard deviation} within each GRPO group, shown separately for attempts 1--4.}
        \label{fig:app_maze9_ver4_reward_std}
    \end{subfigure}
    \hfill
    \begin{subfigure}[t]{0.32\textwidth}
        \centering
        \includegraphics[width=\linewidth]{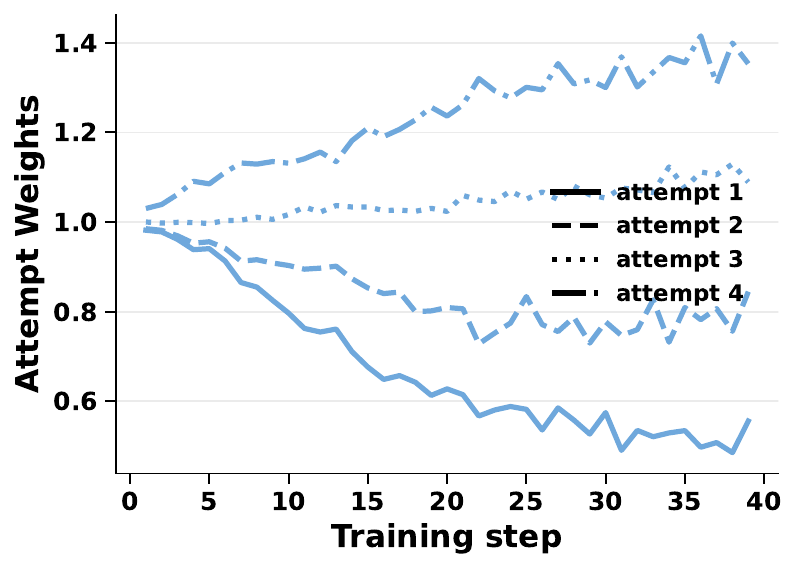}
        \caption{\textbf{Attempt weights} used by CAL-GRPO (future-only reweighting), for attempts 1--4 over training.}
        \label{fig:app_maze9_ver4_weights}
    \end{subfigure}

    \caption{\textbf{Additional metrics for Ver@4 training on 9$\times$9 Maze task.} (a) Response length dynamics per attempt. (b) Per-attempt reward variation used during the per-attempt normalization. (c) Learned future-only attempt reweighting in CAL-GRPO.}
    \label{fig:app_maze9_ver4_diagnostics}
\end{figure*}

\begin{figure*}[t!]
    \centering
    \begin{subfigure}[t]{0.32\textwidth}
        \centering
        \includegraphics[width=\linewidth]{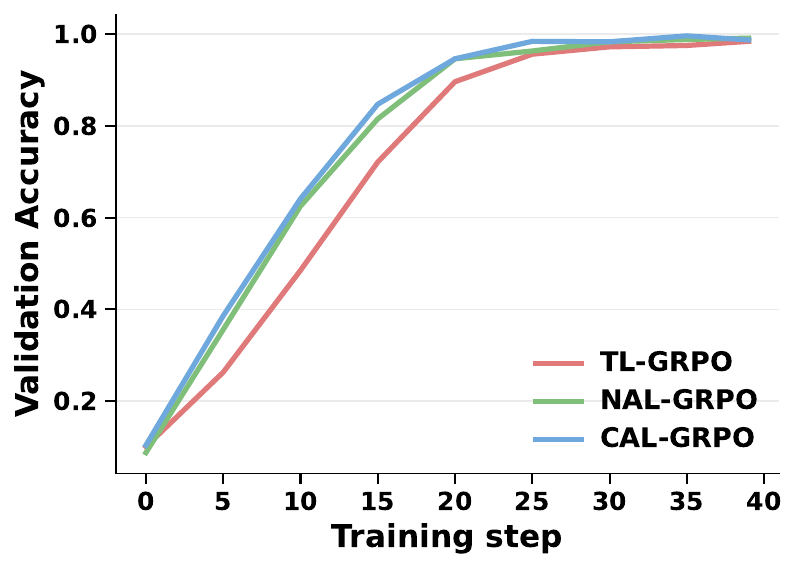}
        \caption{\textbf{Ver@4 validation accuracy on 7$\times$7 Maze} across training; CAL-GRPO achieves the highest Ver@4, with NAL-GRPO improving over TL-GRPO.}
        \label{fig:maze7_ver4_acc}
    \end{subfigure}
    \hfill
    \begin{subfigure}[t]{0.32\textwidth}
        \centering
        \includegraphics[width=\linewidth]{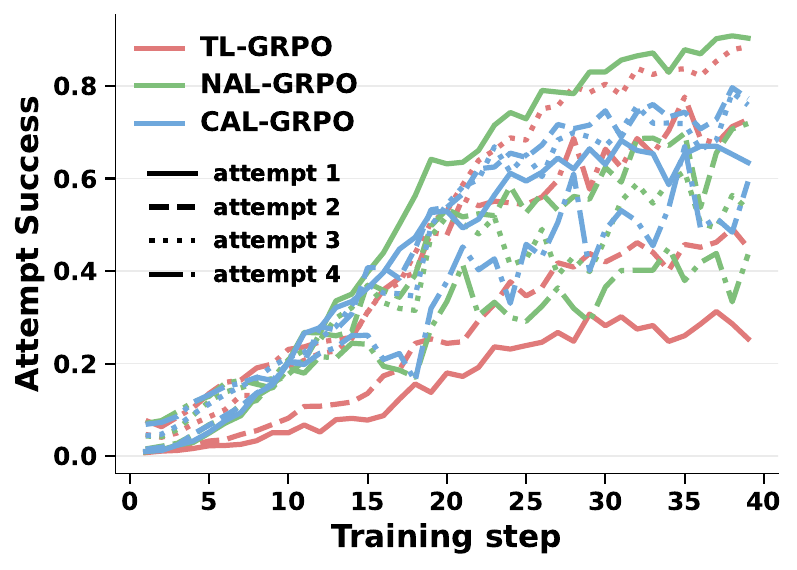}
        \caption{\textbf{Attempt-wise success.} Line styles denote attempts 1--4. Curves show success on attempt $i$ given all earlier attempts failed; NAL-GRPO boosts early success while CAL-GRPO improves correction.}
        \label{fig:maze7_ver4_attempt_success}
    \end{subfigure}
    \hfill
    \begin{subfigure}[t]{0.32\textwidth}
        \centering
        \includegraphics[width=\linewidth]{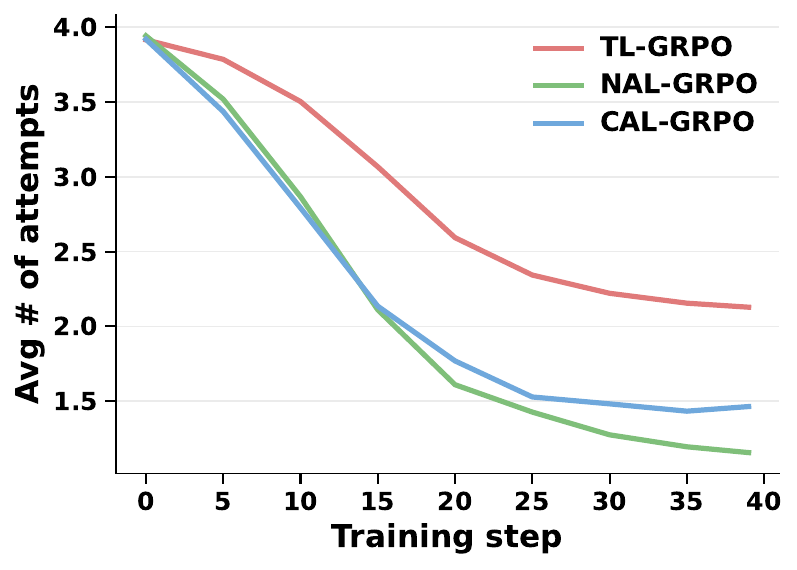}
        \caption{\textbf{Average attempts} under Ver@4; NAL/CAL reduce the number of attempts relative to TL-GRPO.}
        \label{fig:maze7_ver4_avg_attempts}
    \end{subfigure}

    \caption{\textbf{Optimizing the Ver@4 objective on 7$\times$7 Maze with GRPO variants.} We compare Trajectory-Level GRPO (TL-GRPO), Naive Attempt-Level GRPO (NAL-GRPO), and Calibrated Attempt-Level GRPO (CAL-GRPO). \textbf{(a)} CAL-GRPO attains the best Ver@4. \textbf{(b)} Improvements reflect stronger attempt-wise performance, with CAL-GRPO emphasizing correction on later attempts. \textbf{(c)} NAL/CAL reduce the average number of attempts compared to TL-GRPO.}
    \label{fig:maze7_ver4_estimators}
\end{figure*}

\begin{figure*}[t!]
    \centering
    \begin{subfigure}[t]{0.32\textwidth}
        \centering
        \includegraphics[width=\linewidth]{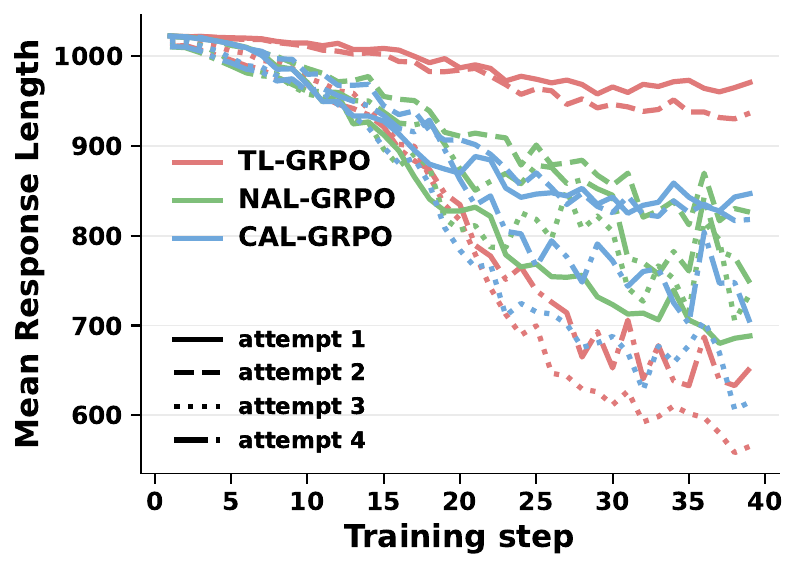}
        \caption{\textbf{Mean response length} per attempt vs.\ training step (baselines labeled in legend: TL-GRPO, NAL-GRPO, CAL-GRPO).}
        \label{fig:app_maze7_ver4_mean_len}
    \end{subfigure}
    \hfill
    \begin{subfigure}[t]{0.32\textwidth}
        \centering
        \includegraphics[width=\linewidth]{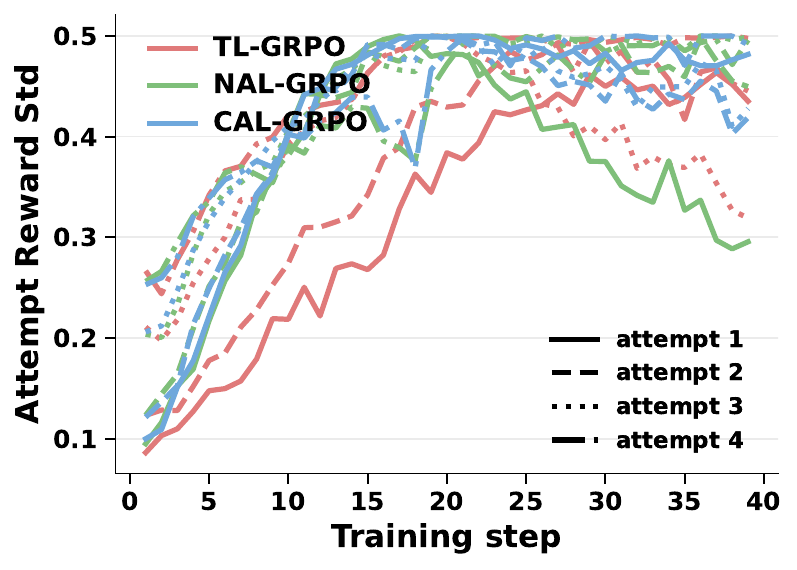}
        \caption{\textbf{Attempt reward standard deviation} within each GRPO group, shown separately for attempts 1--4.}
        \label{fig:app_maze7_ver4_reward_std}
    \end{subfigure}
    \hfill
    \begin{subfigure}[t]{0.32\textwidth}
        \centering
        \includegraphics[width=\linewidth]{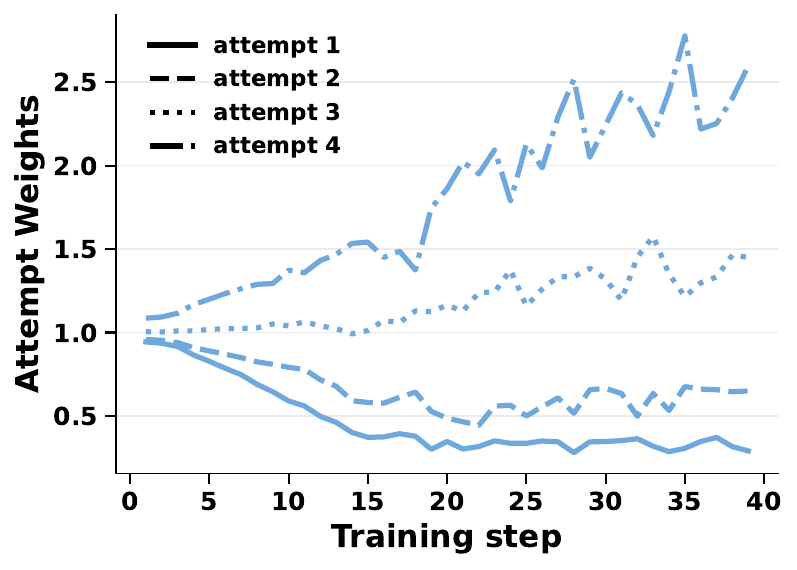}
        \caption{\textbf{Attempt weights} used by CAL-GRPO (future-only reweighting), for attempts 1--4 over training.}
        \label{fig:app_maze7_ver4_weights}
    \end{subfigure}

    \caption{\textbf{Additional metrics for Ver@4 training on 7$\times$7 Maze task.} (a) Response length dynamics per attempt. (b) Per-attempt reward variation used during the per-attempt normalization. (c) Learned future-only attempt reweighting in CAL-GRPO.}
    \label{fig:app_maze7_ver4_diagnostics}
\end{figure*}

\begin{figure*}[t!]
    \centering
    \begin{subfigure}[t]{0.32\textwidth}
        \centering
        \includegraphics[width=\linewidth]{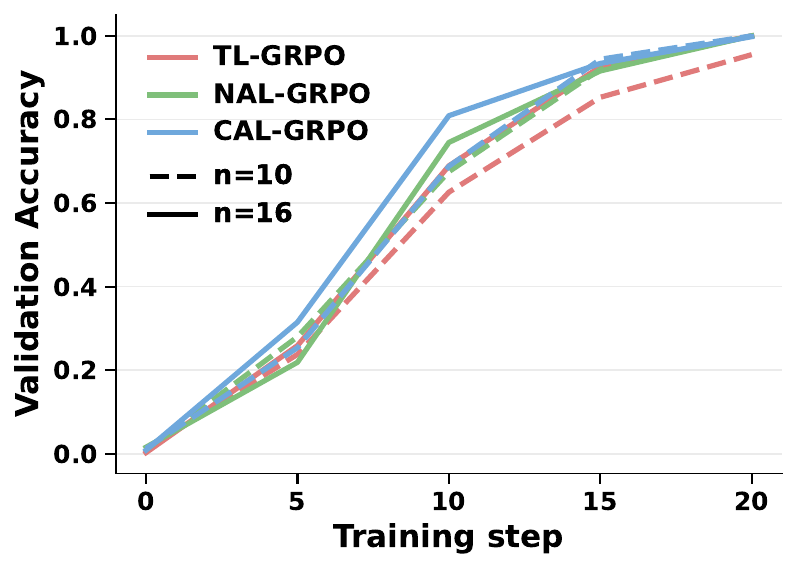}
        \caption{\textbf{Ver@2 validation accuracy on Maze} for rollout group sizes $N{=}10$ and $N{=}16$. Solid: $N{=}16$. Dashed: $N{=}10$. Larger groups improve sample efficiency, especially for CAL-GRPO early in training.}
        \label{fig:maze_ver2_groupsize_acc}
    \end{subfigure}
    \hfill
    \begin{subfigure}[t]{0.32\textwidth}
        \centering
        \includegraphics[width=\linewidth]{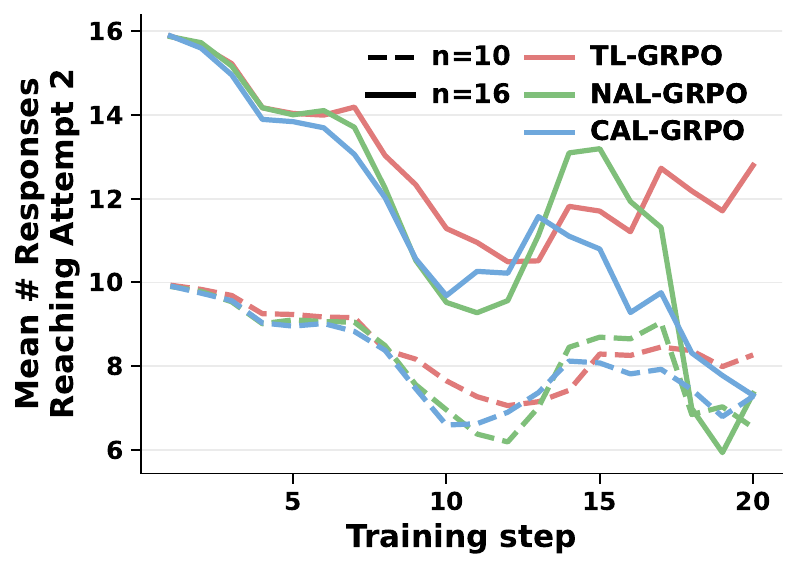}
        \caption{\textbf{Mean number of responses reaching attempt 2.} Solid: $N{=}16$. Dashed: $N{=}10$. Larger groups retain more attempt-2 samples in the current rollout group, improving the stability of per-attempt normalization and CAL-GRPO's weight estimates.}
        \label{fig:maze_ver2_groupsize_reached}
    \end{subfigure}
    \hfill
    \begin{subfigure}[t]{0.32\textwidth}
        \centering
        \includegraphics[width=\linewidth]{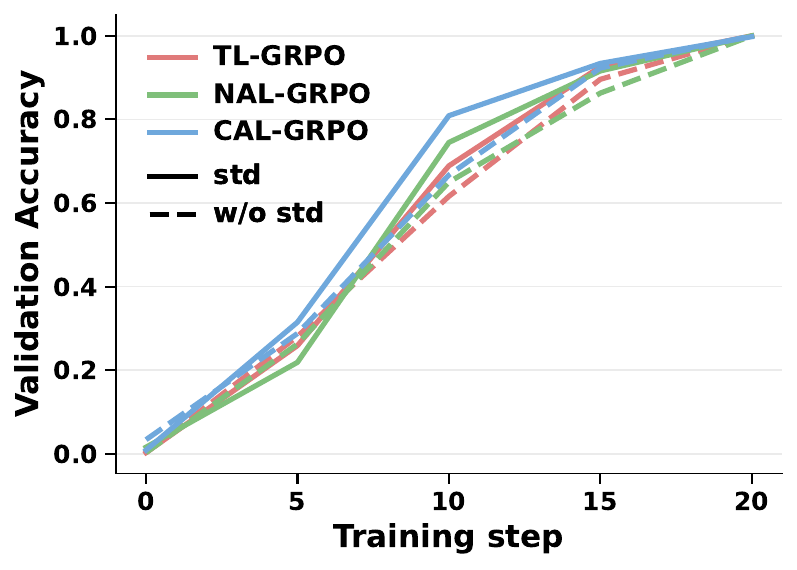}
        \caption{\textbf{Standard deviation normalization ablation at $N{=}16$.} Solid: with within-group standard deviation scaling. Dashed: without standard deviation scaling. Standard deviation normalization improves sample efficiency for all three methods.}
        \label{fig:maze_ver2_stdnorm}
    \end{subfigure}

    \caption{\textbf{Effect of rollout group size and standard deviation normalization on Ver@2 training for Maze.} We compare GRPO group sizes $N{=}10$ and $N{=}16$ for TL-GRPO, NAL-GRPO, and CAL-GRPO, and also compare the standard GRPO normalization against a variant without standard deviation scaling. \textbf{(a)} Larger groups improve early Ver@2, with the most noticeable gain for CAL-GRPO. \textbf{(b)} Since CAL-GRPO estimates attempt-wise weights from the current rollout group, larger $N$ provides more trajectories that reach attempt 2 and therefore yields more accurate estimates of the conditional success rate of attempt 2. \textbf{(c)} Standard deviation normalization improves learning for all three estimators and is used in our experiments. 
    }
    \label{fig:maze_ver2_groupsize}
\end{figure*}

%% file: sections/appendix/equivalence_theorem_proof.tex
\section{Proofs Section \ref{sect:main_results}}\label{app:main_results}
Fix a problem instance $(x,y)$ and suppress $(x,y)$ in notation.
Let $\tau$ denote a (random) multi-attempt rollout with stopping time $\mathcal T\in\{1,\dots,K\}$ and
terminal Verification@$K$ reward $R:=r_{\mathcal T}\in\{0,1\}$.
For attempt $i$, let $s_{i-1}$ be the pre-attempt state and $y_i\sim \pi_\theta(\cdot\mid s_{i-1})$.
Define the attempt-$i$ score (for the whole attempt output) as
\[
g_i \;:=\; \nabla_\theta \log \pi_\theta(y_i\mid s_{i-1})\cdot \mathbf 1\{\mathcal T\ge i\},
\qquad i\in[K],
\]
so $g_i\equiv 0$ on trajectories that terminate before attempt $i$.
Let $\mathcal F_i := \sigma(x,y_1,r_1,\dots,y_i,r_i)$ be the trajectory filtration up to attempt $i$.
Define the Doob conditional success probabilities
\[
Q_i \;:=\; \E\!\left[R \mid \mathcal F_i\right],\qquad i=0,1,\dots,K.
\]
(Under a deterministic verifier, $Q_i = r_i + (1-r_i)V_{\theta,i}(s_i)$ on $\{\mathcal T\ge i\}$, where
$V_{\theta,i}(s_i)=\P(\exists\,t\in\{i+1,\dots,K\}: r_t=1\mid s_i)$ and $V_{\theta,K}(\cdot)=0$.)

Now draw a group of $N\ge 2$ i.i.d.\ trajectories $\tau_1,\dots,\tau_N\sim p_\theta(\cdot\mid x)$.
Write $R_n$ and $Q_{n,i}$ for the corresponding terminal reward and Doob conditional success probability.
Let $\mathcal S_i := \{n\in[N]:\mathcal T(\tau_n)\ge i\}$ be the (random) set of trajectories that reach attempt $i$.
For notational convenience, define $M_i:=|\mathcal S_i|$ and assume $M_i\ge 2$ almost surely (or restrict to events where $M_i\ge 2$).
\begin{theorem}[Unbiasedness and strict Rao--Blackwell variance reduction at each attempt]
\label{thm:rb_per_attempt}
Fix $(x,y)$ and $N\ge 2$.
Define the trajectory-level (TL) and attempt-level (AL) RLOO estimators
\[
\widehat G_{\mathrm{RLOO}}^{\mathrm{TL}}(\theta;x,y)
~:=~
\sum_{i=1}^K \frac{M_i}{N}\cdot \widehat Z_i^{\mathrm{TL}},
\qquad
\widehat G_{\mathrm{RLOO}}^{\mathrm{AL}}(\theta;x,y)
~:=~
\sum_{i=1}^K \frac{M_i}{N}\cdot \widehat Z_i^{\mathrm{AL}},
\]
where, for each attempt index $i$,
\begin{align*}
\widehat Z_i^{\mathrm{TL}}
&:= \frac{1}{M_i}\sum_{n\in \mathcal S_i}
A_{n,i}^{\mathrm{TL}}\, g_{n,i},
\qquad
A_{n,i}^{\mathrm{TL}}
~:=~
R_n \;-\; \frac{1}{M_i-1}\sum_{\substack{m\in\mathcal S_i\\ m\neq n}} R_m,\\[2mm]
\widehat Z_i^{\mathrm{AL}}
&:= \frac{1}{M_i}\sum_{n\in \mathcal S_i}
A_{n,i}^{\mathrm{AL}}\, g_{n,i},
\qquad
A_{n,i}^{\mathrm{AL}}
~:=~
Q_{n,i} \;-\; \frac{1}{M_i-1}\sum_{\substack{m\in\mathcal S_i\\ m\neq n}} Q_{m,i}.
\end{align*}
Then:
\begin{enumerate}[label = \roman*]
\item Both estimators are unbiased for the Verification@$K$ policy gradient:
\[
\E\!\left[\widehat G_{\mathrm{RLOO}}^{\mathrm{TL}}(\theta;x,y)\right]
~=~
\E\!\left[\widehat G_{\mathrm{RLOO}}^{\mathrm{AL}}(\theta;x,y)\right]
~=~
\nabla_\theta \rho_\theta^{\mathrm{VK}}(x,y).
\]
\item 
Let $\mathcal G_i := \sigma\!\left(\mathcal F_{n,i}:n\in[N]\right)$ be the group filtration up to attempt $i$.
Then, for every $i\in[K]$,
\[
\E\!\left[\widehat Z_i^{\mathrm{TL}} \,\middle|\, \mathcal G_i\right]
~=~
\widehat Z_i^{\mathrm{AL}},
\qquad\text{and hence}\qquad
\Var(\widehat Z_i^{\mathrm{TL}})
~=~
\Var(\widehat Z_i^{\mathrm{AL}}) \;+\; \E\!\left[\Var(\widehat Z_i^{\mathrm{TL}}\mid \mathcal G_i)\right].
\]
In particular, $\Var(\widehat Z_i^{\mathrm{TL}})\ge \Var(\widehat Z_i^{\mathrm{AL}})$ for all $i$, and the inequality is \emph{strict} for a given $i$ whenever
$\P\!\left(\Var(R_1\mid \mathcal F_{1,i})>0\right)>0$ and $\P\!\left(\|g_{1,i}\|_2>0\right)>0$.
\end{enumerate}
\end{theorem}

\begin{proof}
\textbf{Part (i)}
Fix $i\in[K]$.
Because the trajectories are i.i.d.\ across $n$, and because leave-one-out centering preserves unbiasedness,
it suffices to show that for a single trajectory
\[
\E\!\left[\mathbf 1\{\mathcal T\ge i\}\, Q_i\, g_i\right]
~=~
\E\!\left[\mathbf 1\{\mathcal T\ge i\}\, R\, g_i\right],
\]
and that the sum over $i$ equals $\nabla_\theta \rho_\theta^{\mathrm{VK}}(x,y)$.
The first equality is the tower property:
since $g_i$ is $\mathcal F_i$-measurable and $Q_i=\E[R\mid \mathcal F_i]$,
\[
\E\!\left[\mathbf 1\{\mathcal T\ge i\}\, R\, g_i\right]
~=~
\E\!\left[\E\!\left[\mathbf 1\{\mathcal T\ge i\}\, R\, g_i \mid \mathcal F_i\right]\right]
~=~
\E\!\left[\mathbf 1\{\mathcal T\ge i\}\, g_i \, \E[R\mid \mathcal F_i]\right]
~=~
\E\!\left[\mathbf 1\{\mathcal T\ge i\}\, Q_i\, g_i\right].
\]
Summing over $i$ gives
\[
\sum_{i=1}^K \E\!\left[\mathbf 1\{\mathcal T\ge i\}\, R\, g_i\right]
~=~
\E\!\left[R \sum_{i=1}^K g_i\right]
~=~
\E\!\left[R\, \nabla_\theta \log p_\theta(\tau\mid x)\right]
~=~
\nabla_\theta \E[R]
~=~
\nabla_\theta \rho_\theta^{\mathrm{VK}}(x,y),
\]
where we used $\nabla_\theta\log p_\theta(\tau\mid x)=\sum_{i=1}^K g_i$ (verifier independent of $\theta$)
and the log-derivative identity.
Finally, since the group estimator averages $N$ i.i.d.\ copies and uses leave-one-out baselines
(which are conditionally independent of each centered trajectory),
the expectation is unchanged, proving (i).
ep 
\textbf{Part (ii)}
Fix $i\in[K]$ and consider the group sigma-field $\mathcal G_i=\sigma(\mathcal F_{n,i}:n\in[N])$.
Note that $\mathcal S_i$ and $M_i$ are $\mathcal G_i$-measurable (reachability is determined by $(r_{n,1},\dots,r_{n,i-1})$),
and $g_{n,i}$ is $\mathcal F_{n,i}$-measurable, hence also $\mathcal G_i$-measurable.
Moreover, for each $n$,
\[
\E[R_n \mid \mathcal G_i] \;=\; \E[R_n \mid \mathcal F_{n,i}] \;=\; Q_{n,i},
\]
because $R_n$ depends only on trajectory $n$ and trajectories are independent across $n$.
Therefore, conditioning $\widehat Z_i^{\mathrm{TL}}$ on $\mathcal G_i$ replaces each $R_n$ in the leave-one-out
advantage by $Q_{n,i}$ while leaving all other terms fixed:
\[
\E\!\left[\widehat Z_i^{\mathrm{TL}}\mid \mathcal G_i\right]
=
\frac{1}{M_i}\sum_{n\in\mathcal S_i}
\left(
\E[R_n\mid \mathcal G_i] - \frac{1}{M_i-1}\sum_{\substack{m\in\mathcal S_i\\m\neq n}} \E[R_m\mid \mathcal G_i]
\right) g_{n,i}
=
\widehat Z_i^{\mathrm{AL}}.
\]

Applying the law of total variance yields
\[
\Var(\widehat Z_i^{\mathrm{TL}})
=
\Var\!\Big(\E[\widehat Z_i^{\mathrm{TL}}\mid \mathcal G_i]\Big)
+
\E\!\left[\Var(\widehat Z_i^{\mathrm{TL}}\mid \mathcal G_i)\right]
=
\Var(\widehat Z_i^{\mathrm{AL}})
+
\E\!\left[\Var(\widehat Z_i^{\mathrm{TL}}\mid \mathcal G_i)\right],
\]
proving the non-strict inequality.
Moreover, $\E[\Var(\widehat Z_i^{\mathrm{TL}}\mid \mathcal G_i)]>0$ whenever, with positive probability,
there exists residual randomness in $R$ after conditioning on $\mathcal F_i$ and the attempt-$i$ score is nonzero.
A sufficient (and minimal) condition is:
$\P(\Var(R_1\mid \mathcal F_{1,i})>0)>0$ and $\P(\|g_{1,i}\|_2>0)>0$.
This implies strict inequality for that $ii$.
\end{proof}
\section{Proofs for Section \ref{sect:independent_attemtp}}\label{app:independent_attempt}
\begin{theorem}[CAL is unbiased under attempt-index independence; per-attempt Rao--Blackwell variance reduction]
\label{thm:independent-attempt}
Fix $(x,y)$ and suppress $(x,y)$ in notation. Assume Assumption~\ref{assump:index-only-assumption} holds.
Let $\rhoVK_\theta(x,y):=\E[R]$ denote the Verification@$K$ success probability.

Draw $N\ge 2$ i.i.d.\ trajectories $\tau_1,\dots,\tau_N\sim p_\theta(\cdot\mid x)$ and define
$\mathcal S_i:=\{n\in[N]:\mathcal T(\tau_n)\ge i\}$ and $M_i:=|\mathcal S_i|$, assuming $M_i\ge 2$ a.s.
For each $n\in\mathcal S_i$, define the leave-one-out estimate of the attempt-$i$ success probability
\[
\rho_{n,i}
~:=~
\frac{1}{M_i-1}\sum_{\substack{m\in\mathcal S_i\\ m\neq n}} r_{m,i},
\qquad
A^{\mathrm{CAL}}_{n,i}:=r_{n,i}-\rho_{n,i}.
\]
Define the CAL \emph{future-failure} weight
\begin{equation}
\label{eq:cal_weight_correct}
w^{\mathrm{CAL}}_{n,i}
~:=~
\prod_{j=i+1}^{K}\bigl(1-\rho_{n,j}\bigr),
\qquad i\in[K],
\end{equation}
(where for $j>i$ we interpret $\rho_{n,j}$ as the leave-one-out average over $\mathcal S_j$
excluding $n$ if $n\in\mathcal S_j$ and otherwise as the full average over $\mathcal S_j$; this makes
$\rho_{n,j}$ well-defined for all $n$ and $j$ and agrees with the above definition when $n\in\mathcal S_j$).
Define the per-attempt CAL term
\[
Z_i^{\mathrm{CAL}}
~:=~
\frac{1}{M_i}\sum_{n\in\mathcal S_i}
w^{\mathrm{CAL}}_{n,i}\,A^{\mathrm{CAL}}_{n,i}\,g_{n,i},
\qquad i\in[K],
\]
and the CAL group estimator (in the same scaling as the TL/AL definitions)
\begin{equation}
\label{eq:cal_estimator_scaled}
\widehat{G}_{\mathrm{RLOO}}^{\mathrm{CAL}}(\theta;x,y)
~:=~
\sum_{i=1}^{K}\frac{M_i}{N}\,Z_i^{\mathrm{CAL}}
~=~
\frac{1}{N}\sum_{i=1}^{K}\sum_{n\in\mathcal S_i}
w^{\mathrm{CAL}}_{n,i}\,A^{\mathrm{CAL}}_{n,i}\,g_{n,i}.
\end{equation}
Then:
\begin{align}
\label{eq:cal_unbiased}
\E\!\left[\widehat{G}_{\mathrm{RLOO}}^{\mathrm{CAL}}(\theta;x,y)\right]
~=~
\nabla_\theta \rhoVK_\theta(x,y).
\end{align}
\end{theorem}

\begin{proof}
\textbf{Unbiasedness: }
Let $\mathcal G_i:=\sigma(\mathcal F_{1,i},\dots,\mathcal F_{N,i})$ be the group filtration.

Assumption~\ref{assump:index-only-assumption} implies that there exist policies
$\bigl(\pi_\theta^{(i)}(\cdot\mid x)\bigr)_{i=1}^K$ such that
$p_\theta(y_i\mid s_{1:i-1})=\pi_\theta^{(i)}(y_i\mid x)$.
Hence, conditional on $x$, the attempt outputs $y_1,\dots,y_K$ are independent across $i$,
and so are the correctness indicators $(r_i)_{i=1}^K$.

Define the attempt-$i$ success probability (conditional on reaching attempt $i$)
\[
\rho_i ~:=~ \P(r_i=1\mid \mathcal T\ge i).
\]
Under Assumption~\ref{assump:index-only-assumption}, $\rho_i$ depends only on $(x,i,\theta)$ and not on the
content of the prior history.

Let
\[
u_i ~:=~ \P(r_{i+1}=0,\dots,r_K=0 \mid \mathcal T\ge i)
~=~ \prod_{j=i+1}^K (1-\rho_j),
\qquad i\in[K],
\]
(with $u_K:=1$ by the empty-product convention).
Then, on $\{\mathcal T\ge i\}$, the Doob conditional success probability satisfies
\begin{equation}
\label{eq:Qi_affine}
Q_i
~=~
\E[R\mid \mathcal F_i]
~=~
1-u_i(1-r_i)
~=~
(1-u_i)+u_i r_i.
\end{equation}

We claim that for each $n\in[N]$ and $i\in[K]$,
\begin{equation}
\label{eq:weight_condexp}
\E\!\left[w^{\mathrm{CAL}}_{n,i}\mid \mathcal G_i\right]
~=~
u_i
~=~
\prod_{j=i+1}^K (1-\rho_j).
\end{equation}
To see this, define $W_{n,t}:=\prod_{j=t}^K (1-\rho_{n,j})$ for $t\in\{i+1,\dots,K\}$, so that
$w^{\mathrm{CAL}}_{n,i}=W_{n,i+1}$.
We prove by backward induction that
\[
\E[W_{n,t}\mid \mathcal G_{t-1}]
~=~
\prod_{j=t}^K (1-\rho_j),
\qquad t=i+1,\dots,K.
\]
For the base case $t=K$, conditional on $\mathcal G_{K-1}$ the set $\mathcal S_K$ is fixed and, for each
$m\in\mathcal S_K$, $r_{m,K}$ is an i.i.d.\ Bernoulli$(\rho_K)$ draw independent of $\mathcal G_{K-1}$.
Thus $\rho_{n,K}$ is an average of i.i.d.\ Bernoulli$(\rho_K)$ variables (with or without leaving out $n$),
and therefore $\E[\rho_{n,K}\mid \mathcal G_{K-1}]=\rho_K$, i.e.\ $\E[1-\rho_{n,K}\mid \mathcal G_{K-1}]=1-\rho_K$,
which is exactly the desired identity at $t=K$.

For the induction step, assume the identity holds at $t+1$:
$\E[W_{n,t+1}\mid \mathcal G_t]=\prod_{j=t+1}^K(1-\rho_j)$.
Then, by the tower property,
\begin{align*}
\E[W_{n,t}\mid \mathcal G_{t-1}]
&=
\E\!\left[\E\!\left[(1-\rho_{n,t})W_{n,t+1}\mid \mathcal G_t\right]\middle|\mathcal G_{t-1}\right]\\
&=
\E\!\left[(1-\rho_{n,t})\,\E[W_{n,t+1}\mid \mathcal G_t]\middle|\mathcal G_{t-1}\right]\\
&=
\Bigl(\prod_{j=t+1}^K(1-\rho_j)\Bigr)\cdot \E[1-\rho_{n,t}\mid \mathcal G_{t-1}].
\end{align*}
The same argument as in the base case yields $\E[\rho_{n,t}\mid \mathcal G_{t-1}]=\rho_t$ (because the
attempt-$t$ outcomes among $\mathcal S_t$ are conditionally i.i.d.\ Bernoulli$(\rho_t)$ given $\mathcal G_{t-1}$),
so $\E[1-\rho_{n,t}\mid \mathcal G_{t-1}]=1-\rho_t$ and therefore
\[
\E[W_{n,t}\mid \mathcal G_{t-1}]
=
\prod_{j=t}^K (1-\rho_j).
\]
Setting $t=i+1$ gives \eqref{eq:weight_condexp}.

Using \eqref{eq:Qi_affine}, the AL leave-one-out advantage at attempt $i$ can be rewritten as
\[
A^{\mathrm{AL}}_{n,i}
~:=~
Q_{n,i}-\frac{1}{M_i-1}\sum_{\substack{m\in\mathcal S_i\\ m\neq n}}Q_{m,i}
~=~
u_i\Bigl(r_{n,i}-\frac{1}{M_i-1}\sum_{\substack{m\in\mathcal S_i\\ m\neq n}}r_{m,i}\Bigr)
~=~
u_i\,A^{\mathrm{CAL}}_{n,i},
\]
since the constant term $(1-u_i)$ cancels in the leave-one-out difference.
Consequently,
\begin{equation}
\label{eq:Zi_AL_affine}
Z_i^{\mathrm{AL}}
~=~
\frac{1}{M_i}\sum_{n\in\mathcal S_i} A^{\mathrm{AL}}_{n,i}g_{n,i}
~=~
u_i\cdot \frac{1}{M_i}\sum_{n\in\mathcal S_i} A^{\mathrm{CAL}}_{n,i}g_{n,i}.
\end{equation}

On the other hand, $A^{\mathrm{CAL}}_{n,i}g_{n,i}$ is $\mathcal G_i$-measurable, so by
\eqref{eq:weight_condexp},
\begin{align}
\E[Z_i^{\mathrm{CAL}}\mid \mathcal G_i]
&=
\E\!\left[\frac{1}{M_i}\sum_{n\in\mathcal S_i} w^{\mathrm{CAL}}_{n,i}A^{\mathrm{CAL}}_{n,i}g_{n,i}\middle|\mathcal G_i\right]\nonumber\\
&=
\frac{1}{M_i}\sum_{n\in\mathcal S_i} A^{\mathrm{CAL}}_{n,i}g_{n,i}\cdot
\E[w^{\mathrm{CAL}}_{n,i}\mid \mathcal G_i]\nonumber\\
&=
u_i\cdot \frac{1}{M_i}\sum_{n\in\mathcal S_i} A^{\mathrm{CAL}}_{n,i}g_{n,i}
~=~
Z_i^{\mathrm{AL}},
\label{eq:ZiCAL_condexp_equals_ALO}
\end{align}
where the last equality uses \eqref{eq:Zi_AL_affine}.
Taking expectations in \eqref{eq:ZiCAL_condexp_equals_ALO} yields
\begin{equation}
\label{eq:ZiCAL_expect_equals_ALO}
\E[Z_i^{\mathrm{CAL}}]=\E[Z_i^{\mathrm{AL}}],\qquad \forall i\in[K].
\end{equation}
From \eqref{eq:cal_estimator_scaled} and \eqref{eq:ZiCAL_expect_equals_ALO},
\[
\E\!\left[\widehat{G}_{\mathrm{RLOO}}^{\mathrm{CAL}}(\theta;x,y)\right]
=
\sum_{i=1}^K \E\!\left[\frac{M_i}{N}Z_i^{\mathrm{CAL}}\right]
=
\sum_{i=1}^K \E\!\left[\frac{M_i}{N}Z_i^{\mathrm{AL}}\right]
=
\E\!\left[\widehat{G}_{\mathrm{RLOO}}^{\mathrm{AL}}(\theta;x,y)\right].
\]
Finally, the standard policy-gradient/Doob argument gives
$\E[\widehat{G}_{\mathrm{RLOO}}^{\mathrm{AL}}]=\nabla_\theta \E[R]=\nabla_\theta \rhoVK_\theta(x,y)$:
indeed,
\[
\nabla_\theta \rhoVK_\theta(x,y)
=
\nabla_\theta \E[R]
=
\E\!\left[R\sum_{i=1}^K g_i\right]
=
\sum_{i=1}^K \E[Q_i g_i],
\]
and the leave-one-out baseline used to form $Z_i^{\mathrm{AL}}$ does not change the expectation because its
multiplier has conditional mean zero.
This proves \eqref{eq:cal_unbiased}.
\end{proof}